\documentclass{article}

\PassOptionsToPackage{numbers, compress}{natbib}

\usepackage[preprint]{neurips_2026}

\usepackage[utf8]{inputenc} 
\usepackage[T1]{fontenc}    
\usepackage{hyperref}       
\usepackage{url}            
\usepackage{booktabs}       
\usepackage{amsfonts}       
\usepackage{nicefrac}       
\usepackage{microtype}      
\usepackage{xcolor}         
\usepackage{amsmath}
\usepackage{algorithm}
\usepackage{algpseudocode}
\usepackage{graphicx} 
\usepackage{amsthm}
\usepackage{amssymb}
\usepackage{subcaption}
\usepackage{caption}

\usepackage{placeins}
\usepackage{wrapfig}

\theoremstyle{plain}

\newtheorem{lemma}{Lemma}
\newtheorem*{assumption*}{Assumption}
\newtheorem*{theorem*}{Theorem}

\theoremstyle{definition}

\theoremstyle{remark}
\newtheorem{remark}{Remark}

\title{RoSHAP: A Distributional Framework and Robust Metric for Stable Feature Attribution\thanks{
Code is available at \url{https://github.com/Lanxin-Xiang/RobustSHAP}.}}

\author{%
  Lanxin Xiang \\
  Department of Statistics \\
  Virginia Tech \\
  \texttt{francesx@vt.edu} \\
  \And
  Liang Shi \\
  Department of Statistics \\
  Virginia Tech Transportation Institute \\
  Virginia Tech \\
  \texttt{sliang@vt.edu} \\
  \And
  Youhui Ye \\
  Department of Statistics \\
  Virginia Tech \\
  \texttt{yye1997@vt.edu} \\
  \And
  Boyu Jiang \\
  Department of Statistics \\
  Virginia Tech \\
  \texttt{boyuj@vt.edu} \\
  \And  
  Dawei Zhou \\
  Department of Computer Science \\
  Virginia Tech \\
  \texttt{zhoud@vt.edu} \\
  \And  
  Feng Guo\thanks{Corresponding author.} \\
  Department of Statistics \\
  Virginia Tech Transportation Institute \\
  Virginia Tech \\
  \texttt{feng.guo@vt.edu}
}

\begin{document}

\maketitle

\begin{abstract}
Feature attribution analysis is critical for interpreting machine learning models and supporting reliable data-driven decisions. 
However, feature attribution measures often exhibit stochastic variation: different train--test splits, random seeds, or model-fitting procedures can produce substantially different attribution values and feature rankings. 
This paper proposes a framework for incorporating stochastic nature of feature attribution and a robust attribution metric, \textbf{RoSHAP}, for stable feature ranking based on the SHAP metric. 
The proposed framework models the distribution of feature attribution scores and estimates it through bootstrap resampling and kernel density estimation.
We show that, under mild regularity conditions, the aggregated feature attribution score is asymptotically Gaussian, which greatly reduces the computational cost of distribution estimation.
The RoSHAP summarizes the distribution of SHAP into a robust feature-ranking criterion that simultaneously rewards features that are active, strong, and stable.
Through simulations and real-data experiments, the proposed framework and RoSHAP outperform standard single-run attribution measures in identifying signal features. 
In addition, models built using RoSHAP-selected features achieve predictive performance comparable to full-feature models while using substantially fewer predictors.  The proposed RoSHAP approach improves the stability and interpretability of machine learning models, enabling reliable and consistent insights for analysis.
\end{abstract}

\section{Introduction}
\label{sec:intro}
Feature attribution provides a quantitative basis for interpreting machine learning (ML) models by assigning contribution scores to individual predictors. 
For many modern ML models, however, these scores cannot be directly derived from the model structure and must instead be estimated using post hoc attribution methods.  
This task becomes especially challenging when the data are high-dimensional, noisy, and correlated, or when the predictor--response relationship is complex. 
In such settings, attribution results may be difficult to interpret and may not provide a stable basis for scientific or practical conclusions.

A large body of work has developed feature attribution methods for explaining model predictions, including SHapley Additive exPlanations (SHAP) \citep{lundberg2017unified}, Local Interpretable Model-agnostic Explanations (LIME) \citep{ribeiro2016should}, DeepLIFT \citep{shrikumar2017learning}, and information-gain-based measures. 
These metrics capture different aspects of feature contribution, each with distinct theoretical properties and practical trade-offs. 
SHAP, for instance, produces additive instance-level attribution scores grounded in cooperative game theory and is the unique additive attribution method satisfying local accuracy, missingness, and consistency \citep{lundberg2017unified}.
These guarantees have made SHAP a popular choice across diverse domains, including medicine \citep{bhattarai2024explainable, ponce2024practical, vimbi2024interpreting, liu2022diagnosis}, natural language processing \citep{mosca2022shap}, and spatial data analysis \citep{li2022extracting}.

While existing feature attribution methods are valuable, challenges remain in how attribution scores are estimated and interpreted.  
A standard workflow splits the data into training and test sets, fits a model, and computes attribution values from the trained model.
Each step can introduce stochastic variation, including random data partitioning, random initialization, and stochastic optimization. 
Consequently, attribution scores are not fixed quantities; different train--test splits or random seeds can yield substantially different feature rankings \citep{goldwasser2024statistical}. 
This uncertainty is especially pronounced in high-dimensional settings such as genomics, where the number of features far exceeds the number of samples.
Figure~\ref{fig:golub_seed_shap} illustrates this issue: the top-three SHAP-ranked genes for cancer classification vary across train--test splits, and their attribution magnitudes also differ substantially. Such randomness provides an unreliable basis for costly downstream gene testing.

\begin{figure}[ht]
    \centering
    \begin{subfigure}{0.32\linewidth}
        \centering
        \includegraphics[width=\linewidth]{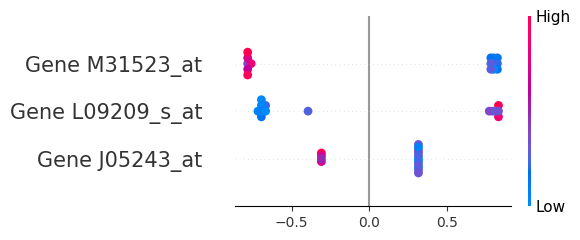}
        \caption{Round 1}
        \label{fig:shap_demo_seed666}
    \end{subfigure}
    \hfill
    \begin{subfigure}{0.32\linewidth}
        \centering
        \includegraphics[width=\linewidth]{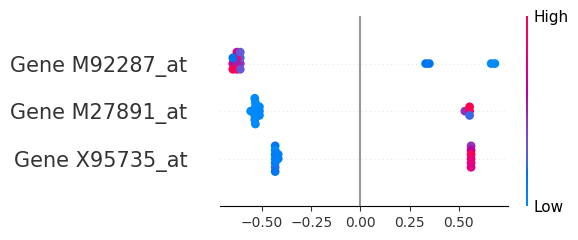}
        \caption{Round 2}
        \label{fig:shap_demo_seed42}
    \end{subfigure}
    \hfill
    \begin{subfigure}{0.32\linewidth}
        \centering
        \includegraphics[width=\linewidth]{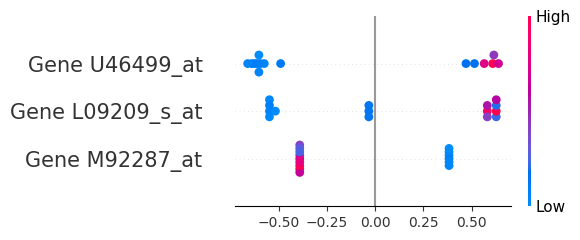}
        \caption{Round 3}
        \label{fig:shap_demo_seed33}
    \end{subfigure}
    \caption{Top SHAP-ranked features in the Golub data across different train--test sets. Although all three runs achieve perfect macro F1, the selected top features differ substantially.}
    \label{fig:golub_seed_shap}
\end{figure}

We propose a bootstrap-based framework for estimating the distribution of feature attribution scores. 
Based on this framework, we introduce \textbf{RoSHAP} (Robust SHAP), a robust attribution metric for stable feature ranking. 
The framework repeatedly resamples the data and refits the model to capture variability in feature attribution scores. 
The resulting attribution scores are summarized through kernel density estimation, yielding an estimated distribution that reflects both attribution magnitude and variability. 
We further derive a Gaussian approximation for the aggregated feature-level attribution score under a mild regularity condition, substantially reducing the computational cost of distribution estimation. 

RoSHAP summarizes the estimated attribution distribution through three components: activity, strength, and stability. 
Activity measures the proportion of nonzero attribution values, strength measures the median magnitude of the nonzero attribution values, and stability measures the signal-to-noise ratio of the attribution distribution. 
By rewarding features that contribute frequently, have large attribution magnitudes, and remain stable across resampling runs, RoSHAP provides a robust basis for feature ranking and interpretation.


Results on both simulated and real-world datasets show that RoSHAP yields stable interpretations across regression and classification tasks and across diverse data modalities.
RoSHAP effectively identifies the most influential features, and models trained using only these selected features achieve performance comparable to that of models trained on the full feature set.

The remainder of this paper is organized as follows. Section \ref{sec:literature} reviews related work on SHAP and other feature contribution methods. Section \ref{sec:simulation} reports simulation studies under different settings. Section \ref{sec:example} presents three motivating examples and performance comparisons.  Section \ref{sec:discussion} concludes with a discussion of the main findings and future directions.

\section{Related Works}
\label{sec:literature}
\paragraph{Feature attribution and additive attribution.}
Feature attribution is widely used to interpret predictive models.
For classical linear models, \(p\)-values, \(t\)-tests, and \(F\)-tests provide feature-level inference under parametric assumptions.
For ML models, common importance measures include gain-, impurity-, and entropy-based scores, as well as post hoc methods such as SHAP \citep{lundberg2017unified}, LIME \citep{ribeiro2016should}, and DeepLIFT \citep{shrikumar2017learning}.
SHAP is grounded in cooperative game theory, treating predictors as players and the model prediction as the payoff; each feature is assigned its average marginal contribution over all coalitions.
This additive formulation makes SHAP broadly applicable across model classes.
TreeSHAP \citep{lundberg2020local} enables efficient computation for tree ensembles such as random forests, XGBoost \citep{chen2016xgboost}, and LightGBM \citep{ke2017lightgbm}, while KernelSHAP and DeepSHAP extend SHAP to general black-box and deep learning models.

\paragraph{SHAP stability and prediction uncertainty.}
Growing literature examines how feature attributions vary under data and model randomness.
\citet{shaikhina2021effects} showed that predictive uncertainty can propagate to explanation uncertainty, producing unstable attribution scores, especially for out-of-distribution samples.
Other studies use repeated resampling or nested cross-validation to summarize SHAP values across folds through averages, selection frequencies, or rank distributions \citep{scheda2022explanations, halabi2026systematic, goldwasser2024statistical}.
Related work explains predictive uncertainty: \citet{watson2023explaining} extended SHAP with information-theoretic value functions to attribute epistemic and aleatoric uncertainty.
However, most approaches focus on averages or rankings rather than the full distributional structure of attribution scores.

\paragraph{SHAP for feature selection.}
SHAP is also used for feature selection, typically by ranking variables by mean absolute SHAP values and retaining the top-ranked subset. 
This strategy has been adopted across application domains \citep{marcilio2020explanations, gebreyesus2023machine, shen2025data}. 
Comparative studies have evaluated SHAP-based feature selection against built-in model importance measures and other screening rules \citep{wang2024feature}. 
While often effective, these approaches remain largely magnitude-based, with limited attention to resampling instability, exact zeros, and strong predictor correlation.

Overall, prior work has established SHAP as a useful tool for interpretation, uncertainty analysis, and feature selection. 
However, existing approaches mainly focus on average attribution magnitude, ranking stability, or predictive uncertainty, rather than the full empirical distribution of attribution scores across samples and resamples. 
This gap motivates the framework proposed in this paper.

\section{Methodology}
\label{sec:method}


To account for the stochastic nature of feature attribution, we adopt a distributional view of feature attributions and introduce a new metric called RoSHAP. Since most additive attribution methods result from complex calculation procedure, the underlying distribution is not available in closed-form. 
We propose to use bootstrapping to approximate these distributions, and then estimate the resulting empirical distribution via kernel density estimation (KDE).
Under mild regularity conditions, the attribution distribution can be approximated by a Gaussian distribution, reducing the computational burden associated with bootstrapping.
RoSHAP then quantifies feature attribution by jointly accounting for inactivity, signal, and noise level. The overall pipeline is illustrated in Figure~\ref{fig:flowchart}.
\vspace{-1em}
\begin{figure}[ht]
    \centering    \includegraphics[width=0.67\linewidth]{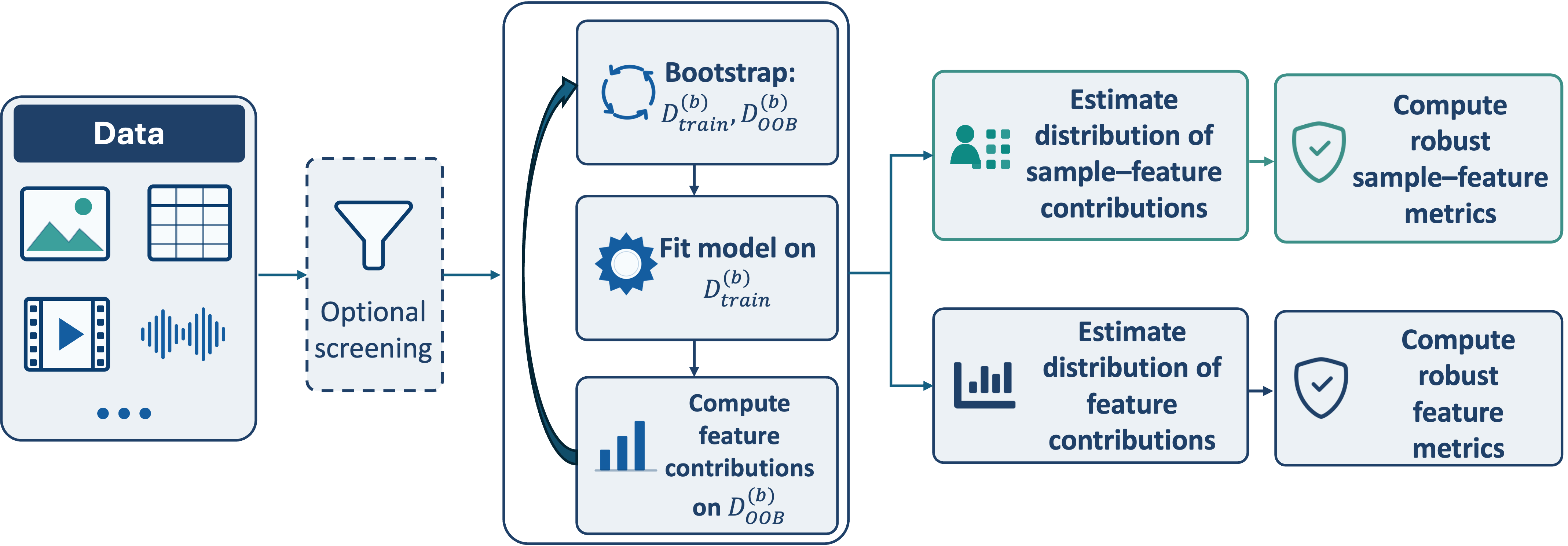}
    \caption{Overall framework for distributional feature attribution estimation.}
    \label{fig:flowchart}
\end{figure}

\clearpage
\begin{wrapfigure}[12]{r}{0.5\textwidth}
\begin{minipage}{0.5\textwidth}
\footnotesize

\hrule
\vspace{0.3em}
\refstepcounter{algorithm}
\noindent\textbf{Algorithm~\thealgorithm.} Bootstrap-based SHAP distribution estimation
\label{alg:data}
\vspace{0.3em}
\hrule
\vspace{0.5em}

\begin{algorithmic}[1]
\Require Data \(\mathcal D\); bootstrap runs \(B\)
\Ensure Sample- and feature-level SHAP distributions
\State Preprocess \(\mathcal D\); optionally screen predictors
\For{$b=1,\ldots,B$}
    \State Draw \(\mathcal D^{(b)}_{\mathrm{train}}\) from \(\mathcal D\)
    \State Define the out-of-bag set \(\mathcal D^{(b)}_{\mathrm{oob}}=\mathcal D\setminus \mathcal D^{(b)}_{\mathrm{train}}\)
    \State Fit model \(\hat f^{(b)}\) on \(\mathcal D^{(b)}_{\mathrm{train}}\)
    \State Compute OOB SHAP \(T^{(b)}_{ij}\)
\EndFor
\State Summarize \(U_{j}^{(b)}\) by \(\sum_i |T_{ij}|\)
\State Estimate distribution for \(T_{ij}\) and \(U_{j}\)
\end{algorithmic}
\vspace{0.3em}
\hrule
\end{minipage}
\vspace{-1em}
\end{wrapfigure}

\paragraph{Bootstrap-based Estimation of Feature Attributions.}
Let \(\mathcal{D}_n=\{(x^{(i)},y_i): i=1,\dots,n\}\) denote a labeled dataset with \(n\) observations and \(p\) features, where \(x^{(i)}=(x_{i1},\dots,x_{ip})^\top \in \mathbb{R}^p\) and \(y_i\) is the response. 
Let \(\hat f(\cdot)\) denote the fitted model, with prediction \(\hat f(x^{(i)})\) for observation \(i\). 
For each observation \(i\) and feature \(j\), denote by \(T_{ij}\) the contribution of feature \(j\) at \(x^{(i)}\).

We estimate the distribution of attribution values by bootstrapping. 
By repeatedly resampling the training data, refitting the model, and recomputing attributions, we obtain multiple realizations of \(T_{ij}\), which serve as samples from the underlying distribution.
Algorithm~\ref{alg:data} summarizes the bootstrap-based distribution estimation procedure.

\paragraph{Zero-Inflated Modeling and Gaussian Approximation of Feature-level Attributions.}

Our framework applies to any additive attribution method, and we use SHAP as a representative example throughout. In practice, SHAP values often contain exact zeros, especially in tree-based models, where features absent from a prediction path receive no attribution \citep{lundberg2017unified,lundberg2020local}.
To account for this behavior, we model the sample-feature SHAP value \(T_{ij}\) as a zero-inflated mixture,
\[
T_{ij} \sim w_{ij}\,\delta_0 + (1-w_{ij})\,G_{ij},
\]
where \(\delta_0\) is a point mass at zero, \(w_{ij}=\Pr(T_{ij}=0)\), and \(G_{ij}\) denotes the nonzero component of \(T_{ij}\).


We define the feature-level attribution as the aggregated magnitude
\[
U_j=\sum_{i=1}^n |T_{ij}| = \sum_{i=1}^n (1-B_{ij})H_{ij}.
\]
For asymptotic analysis, it is convenient to decompose
\[
|T_{ij}|=(1-B_{ij})H_{ij},
\qquad
B_{ij}\sim \mathrm{Bernoulli}(w_{ij}),
\]
where \(B_{ij}=1\) indicates \(T_{ij}=0\), and \(H_{ij}=|G_{ij}|\) is the nonzero magnitude.
This representation separates sparsity from magnitude while preserving the distribution of \(|T_{ij}|\).
We estimate the distribution of \(U_j\) nonparametrically using KDE. 


When mild regularity conditions hold, the distribution of \(U_j\) can be well approximated by a Gaussian distribution and can therefore be summarized by its mean and standard deviation. This approximation is justified through a Lyapunov central limit theorem argument.


The mean and variance of \(U_j\) are given by
\[
\begin{gathered}
\mu_j := E(U_j)=\sum_{i=1}^n (1-w_{ij})E(H_{ij}),\\[-0.4em]
s_j^2 := \operatorname{Var}(U_j)
=\sum_{i=1}^n (1-w_{ij})\operatorname{Var}(H_{ij})
+\sum_{i=1}^n w_{ij}(1-w_{ij})\{E(H_{ij})\}^2 .
\end{gathered}
\]
\vspace{-0.8em}
\begin{assumption*}
For a fixed feature \(j\), the pairs \((B_{ij},H_{ij})\), \(i=1,\dots,n\), are mutually independent, with \(B_{ij}\perp H_{ij}\) for each \(i\). In addition, for some \(\delta>0\), the Lyapunov condition holds:
\[
\frac{1}{s_j^{2+\delta}}
\sum_{i=1}^n
E\left|
|T_{ij}|-E[|T_{ij}|]
\right|^{2+\delta}
\to 0 .
\]
\end{assumption*}
This assumption ensures that the aggregate contribution \(U_j\) is governed by many small terms rather than a few dominant ones, which is crucial for the Gaussian approximation.
\begin{theorem*}
Under the above assumption,
\[
\frac{U_j-\mu_j}{s_j}\xrightarrow{d}N(0,1).
\]
\end{theorem*}

This result justifies approximating the distribution of \(U_j\) by a Gaussian when $n \to \infty$. In practice, it implies that for features with sufficiently stable aggregated contributions, uncertainty can be effectively summarized using only the mean and variance, substantially reducing the number of bootstrap resamples. A proof is provided in Appendix~\ref{app:proof_clt}.
When the regularity conditions are approximately satisfied, we summarize \(U_j\) using its empirical median and standard deviation.

\paragraph{Robust metric.}
Based on the estimated feature-level attribution distribution, we define several complementary metrics for feature evaluation.
For feature \(j\), define
\[
P_{0j} := \Pr(U_j=0), \qquad
m_j := \operatorname{Median}(U_j\mid U_j>0), \qquad
s_j := \operatorname{SD}(U_j).
\]
The robust feature attribution score, \textbf{RoSHAP}, is defined as
\[
\mathrm{RoSHAP}_j := (1-P_{0j})\frac{m_j^2}{s_j}.
\]
Here, \(P_{0j}\) measures inactivity, while \(\mathrm{RoSHAP}_j\) rewards features that are active, strong, and stable.

\section{Simulation}
\label{sec:simulation}
This simulation study demonstrates that the proposed method provides more robust and interpretable feature contribution estimates. From a distributional perspective, the method helps identify stable signals and distinguish them from noise features selected by chance. 
The results verify that the feature contribution distributions are approximately Gaussian, allowing fewer bootstrap repetitions and substantially reducing computational cost.

\paragraph{Simulation setup.}
A binary classification setting was considered with \(n\) observations and \(d\) predictors, where the first \(s\) predictors were true signals and the remaining \(d-s\) predictors were noise, with \(s<d\). To mimic settings in which a feature is informative at the population level but may contribute little for some individual observations, predictors were generated from zero-inflated Gaussian distributions.

For each signal feature \(j=1,\ldots,s\),
\[
X_j \mid Y=y \sim \pi_{\text{signal}}\delta_0
+
(1-\pi_{\text{signal}})N(\mu_{jy}, \sigma_{\text{signal}}^2),
\]
where \(Y\in\{0,1\}\), \(P(Y=1)=0.5\), and \(\mu_{j1}=-\mu_{j0}\). Thus, signal features differed between the two classes through their nonzero component means.

For noise features \(j=s+1,\ldots,d\), the class-conditional distributions had the same zero-inflated Gaussian form but with zero mean in both classes:
\[
X_j \mid Y=y \sim \pi_{\text{noise}}\delta_0
+
(1-\pi_{\text{noise}})N(0,\sigma_{\text{noise}}^2).
\]
Therefore, noise features had no class-dependent mean shift. Under this design, signal features may be locally inactive when they take the value zero, while remaining globally informative through their class-dependent nonzero means.

In this experiment, \(n\), \(d\), and \(s\) were set to 600, 1000, and 10, respectively. For signal features, \(\sigma_{\text{signal}}=3\) and \(\pi_{\text{signal}}=0.3\). For noise features, \(\sigma_{\text{noise}}=1\) and \(\pi_{\text{noise}}=0.2\). The signal strengths were defined by \(\mu_{j1}=-\mu_{j0}\), with \(\mu_{j1}\) evenly spaced from 1.5 to 0.4 across the ten signal features. An XGBoost classifier was then fitted using the same model parameters across all comparison settings.

\paragraph{Results.}
Figure~\ref{fig:simu_importance_dist} shows the empirical distributions of the importance estimates over 1000 bootstrap runs. Among the selected predictors, the distributions for \(x_1\) and \(x_3\) are approximately Gaussian, while the distribution for \(x_7\) is slightly right-skewed. Overall, the nonzero portions of the distributions are close to Gaussian, supporting the use of Gaussian-based uncertainty summaries with fewer bootstrap repetitions.

\begin{figure}[h]
    \centering
    \begin{subfigure}{0.32\linewidth}
        \centering
        \includegraphics[width=\linewidth]{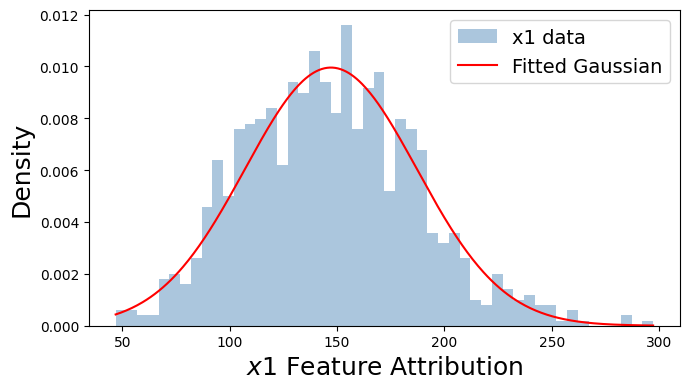}
        \caption{\(x_1\)}
        \label{fig:simu_x1_dist}
    \end{subfigure}
    \hfill
    \begin{subfigure}{0.32\linewidth}
        \centering
        \includegraphics[width=\linewidth]{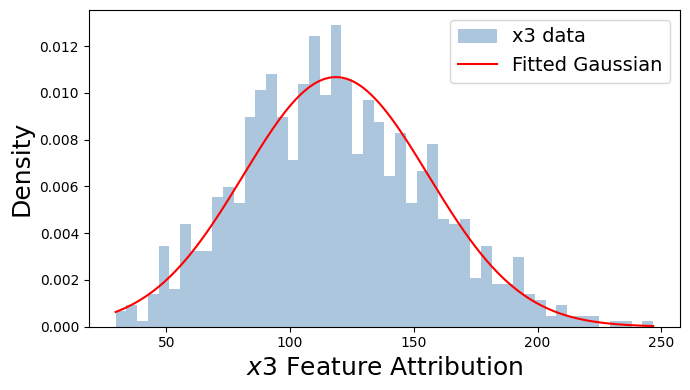}
        \caption{\(x_3\)}
        \label{fig:simu_x3_dist}
    \end{subfigure}
    \hfill
    \begin{subfigure}{0.32\linewidth}
        \centering
        \includegraphics[width=\linewidth]{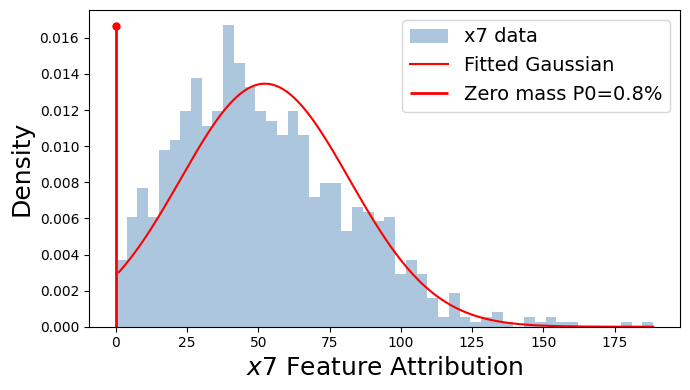}
        \caption{\(x_7\)}
        \label{fig:simu_x7_dist}
    \end{subfigure}

    \caption{Empirical distributions of robust importance estimates over 1000 bootstrap runs with fitted Gaussian distributions for predictors \(x_1\), \(x_3\), and \(x_7\).}
    \label{fig:simu_importance_dist}
    \vspace{-0.5em}
\end{figure}


Three settings were compared: 1000-run bootstrap, 10-run bootstrap, and standard single train--test split SHAP.
As shown in Table~\ref{tab:ranking_comparison}, the 1000-run bootstrap provides the most reliable ranking: eight signal features are recovered within the top 12, and the lowest-ranked signal feature, \(x_9\), is placed 85th among 1000 features based on the robust importance metric. The 10-run bootstrap also performs well, recovering eight signal features among the top 12, although the lowest-ranked signal feature drops to 317th among 1000 features.
The zero percentage values \(P_0\) provide additional insight into stability. Strong signal features have near-zero \(P_0\), indicating consistent contribution across bootstrap runs. In contrast, weaker or less stable features, such as \(x_8\) and \(x_9\), exhibit larger \(P_0\), suggesting more frequent zero contributions.

\begin{table}[ht]
\centering
\tiny
\setlength{\tabcolsep}{1.5pt}
\renewcommand{\arraystretch}{0.85}
\caption{Comparison of feature rankings from bootstrap-based robust importance, single-round SHAP, information gain, XGBoost gain, and LIME screening metrics.}
\label{tab:ranking_comparison}
\resizebox{\textwidth}{!}{%
\begin{tabular}{ccrr|ccrr|ccr|ccr|ccr|ccr}
\hline
\multicolumn{4}{c|}{1000 Runs} &
\multicolumn{4}{c|}{10 Runs} &
\multicolumn{3}{c|}{XGBoost SHAP} &
\multicolumn{3}{c|}{Information Gain} &
\multicolumn{3}{c|}{XGBoost Gain} &
\multicolumn{3}{c}{LIME} \\
\hline
Rank & ID & RoSHAP & \(P_0\) &
Rank & ID & RoSHAP & \(P_0\)&
Rank & ID & Score &
Rank & ID & Score &
Rank & ID & Score &
Rank & ID & Score \\
\hline
1  & \(x_1\)    & 512.66 & 0.00  &
1  & \(x_4\)    & 779.54 & 0.00  &
1  & \(x_2\)    & 35.24  &
1  & \(x_4\)    & 0.090  &
5  & \(x_1\)    & 1.826  &
1  & \(x_3\)    & 0.077  \\

2  & \(x_3\)    & 347.21 & 0.00  &
2  & \(x_1\)    & 767.57 & 0.00  &
2  & \(x_6\)    & 25.86  &
48 & \(x_2\)    & 0.043  &
6  & \(x_2\)    & 1.782  &
2  & \(x_1\)    & 0.075  \\

3  & \(x_4\)    & 339.45 & 0.00  &
3  & \(x_2\)    & 433.42 & 0.00  &
3  & \(x_4\)    & 15.38  &
187 & \(x_5\)   & 0.024  &
10 & \(x_3\)    & 1.696  &
3  & \(x_2\)    & 0.055  \\

4  & \(x_2\)    & 270.75 & 0.00  &
4  & \(x_3\)    & 430.41 & 0.00  &
4  & \(x_3\)    & 15.14  &
235 & \(x_6\)   & 0.021  &
13 & \(x_4\)    & 1.450  &
4  & \(x_4\)    & 0.036  \\

5  & \(x_6\)    & 120.99 & 0.20  &
5  & \(x_6\)    & 271.18 & 0.00  &
12 & \(x_5\)    & 4.12   &
259 & \(x_1\)   & 0.019  &
26 & \(x_6\)    & 1.069  &
5  & \(x_6\)    & 0.035  \\

6  & \(x_5\)    & 80.65  & 0.40  &
6  & \(x_5\)    & 104.68 & 0.00  &
31 & \(x_{10}\) & 0.14   &
281 & \(x_3\)   & 0.017  &
41 & \(x_7\)    & 0.887  &
7  & \(x_7\)    & 0.008  \\

9  & \(x_7\)    & 14.73  & 7.20  &
7  & \(x_7\)    & 60.94  & 0.00  &
333 & \(x_7\)   & 0.00   &
292 & \(x_7\)   & 0.015  &
74 & \(x_{10}\) & 0.686  &
15 & \(x_{10}\) & 0.004  \\

12 & \(x_{10}\) & 10.65  & 11.00 &
10 & \(x_{10}\) & 25.85  & 0.00  &
387 & \(x_8\)   & 0.00   &
329 & \(x_9\)   & 0.012  &
131 & \(x_5\)   & 0.555  &
58 & \(x_9\)    & 0.000  \\

53 & \(x_8\)    & 1.92   & 27.60 &
95 & \(x_9\)    & 2.64   & 30.00 &
669 & \(x_1\)   & 0.00   &
658 & \(x_{10}\)& 0.000  &
159 & \(x_8\)   & 0.494  &
960 & \(x_8\)   & 0.000  \\

85 & \(x_9\)    & 1.08   & 37.00 &
317 & \(x_8\)   & 0.32   & 40.00 &
670 & \(x_9\)   & 0.00   &
744 & \(x_8\)   & 0.000  &
707 & \(x_9\)   & 0.063  &
999 & \(x_5\)   & 0.000  \\
\hline
\end{tabular}%
}
\end{table}

The other metric are substantially less stable. Only four signal features are recovered within the top 12, while four signal features are inactive, including \(x_1\), which has the strongest signal by construction. This highlights that a single SHAP run may fail to identify important predictors, whereas the bootstrap-based robust importance more effectively captures stable signal features.

The alternative metrics are substantially less effective. Single-round SHAP identifies only four signal features within the top 12 and assigns zero or near-zero importance to several true signals, including \(x_1\), the strongest signal by construction. Information gain and gain-based importance also show unreliable rankings, placing several true signals far below noise variables. Despite recovering the first seven signal features reasonably well, LIME treats the remaining three signals as inactive. In contrast, RoSHAP identifies the signal features across repeated model fits more consistently.


\paragraph{Computational comparison.}
TreeSHAP has computational complexity approximately \(O(TLD^2)\), where \(T\) is the number of trees, \(L\) is the maximum number of leaves per tree, and \(D\) is the maximum tree depth \citep{lundberg2020local}.  With \(B\) bootstrap repetitions, the total computation is roughly \(B\,O(TLD^2)\), so reducing \(B\) from 1000 to 10 reduces the SHAP computation by about 100 times.
Each bootstrap sample covers about \(1-(1-1/n)^n \approx 63.2\%\) of the original observations. Across \(B\) runs, the probability that an observation is selected at least once is approximately \(1-0.368^B\), which is nearly 1 when \(B=10\). Thus, fewer bootstrap runs can still provide broad data coverage and useful Gaussian-based uncertainty summaries with much lower runtime.

\section{Experiments}
\label{sec:example}
Four experiments demonstrate the proposed framework across data types. 
The Golub dataset \citep{golub1999molecular} represents high-dimensional small-sample classification, Musk (Version 2) \citep{chapman1994musk2} represents moderate-sized high-dimensional molecular classification, UJIIndoorLoc \citep{ujiindoorloc_310} represents high-dimensional regression, and CIFAR-10 represents 10-class image classification.

For tabular data, feature-level interpretability is evaluated across neural networks, tree-based models including LightGBM, XGBoost, random forest (RF), and CatBoost, and logistic regression (LR). 
Because the true important features are unknown, the proposed feature-selection results are compared with information gain (IG), model-based gain or LR coefficients, LIME, and SHAP. 
Classification performance is evaluated by accuracy, F1 score, average precision (AP), and AUC-ROC; regression performance is evaluated by root mean squared error (RMSE), mean absolute error (MAE), and mean absolute percentage error (MAPE). 
For each method, the top \(k\) selected features are used to refit the same predictive model with fixed parameters. 
All comparisons use the same training--testing split, cross-validation procedure, and held-out test data, so performance differences reflect feature-subset quality rather than model or data-split differences. 
The main text reports results from one representative model, with other models provided in the appendix.
For image data, the framework is applied at the sample level. 
In the CIFAR-10 experiment, a ViT-base classifier with patch-level attribution is used to identify image regions contributing most to the predicted class.

\subsection{Classification: Golub}
The Golub dataset \citep{golub1999molecular} contains \(72\) samples and \(7{,}129\) gene expression probes. The response of interest is leukemia subtype: Acute Lymphoblastic Leukemia (ALL), including both allB and allT, versus Acute Myeloblastic Leukemia (AML). The dataset contains \(47\) ALL samples and \(25\) AML samples.
This experiment studies gene-level feature attribution using an XGBoost classifier. The goal is to identify not only which genes contribute most strongly to prediction, but also which genes contribute most consistently across repeated model fits. The Golub dataset provides a representative high-dimensional, small-sample setting where interpretation is difficult and attribution instability is common.

\begin{wrapfigure}[13]{r}{0.33\textwidth}
    \vspace{-2em}
    \centering
    \includegraphics[width=\linewidth]{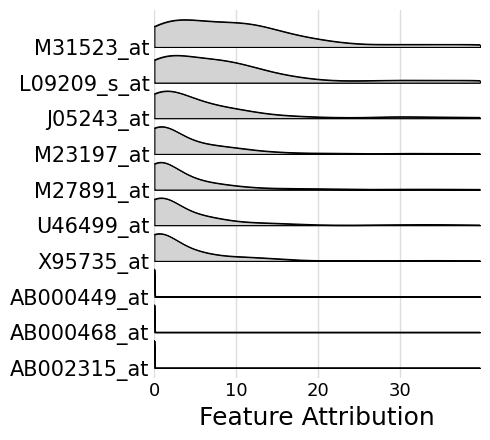}
    \caption{Distribution of top 10 XGBoost feature contribution estimates for the Golub dataset.}
    \label{fig:Golub_XGBoost_top10_dist}
\end{wrapfigure}

Figure~\ref{fig:Golub_XGBoost_top10_dist} shows the distributions of the top 10 genes ranked by RoSHAP metric over 500 bootstrap runs. 
After the eighth-ranked gene, most genes show little or no contribution to cancer prediction.
Compared with the seed-specific SHAP results in Figure~\ref{fig:golub_seed_shap}, the proposed method provides a clearer summary of which genes are consistently important under the fitted XGBoost model.


In this setting, a small number of observations can dominate the variance of, so the Lyapunov condition fails to hold. The distributions are therefore estimated nonparametrically using KDE.

Figure~\ref{fig:golub_XGBoost_comparison} compares feature-selection performance across methods. 
Since experimental validation of candidate genes is often costly, the evaluation focuses on the top 1--15 selected features from each method. 
For each metric, bars show mean performance across candidate feature-set sizes, and error bars show the corresponding variability. 
The proposed method, shown in purple, achieves strong accuracy and F1 score. RoSHAP also maintains comparable AP and AUC-ROC, suggesting stable and reliable feature selection.

\begin{figure}[ht] 
\vspace{-1em}
    \centering
    \includegraphics[width=0.9\linewidth]{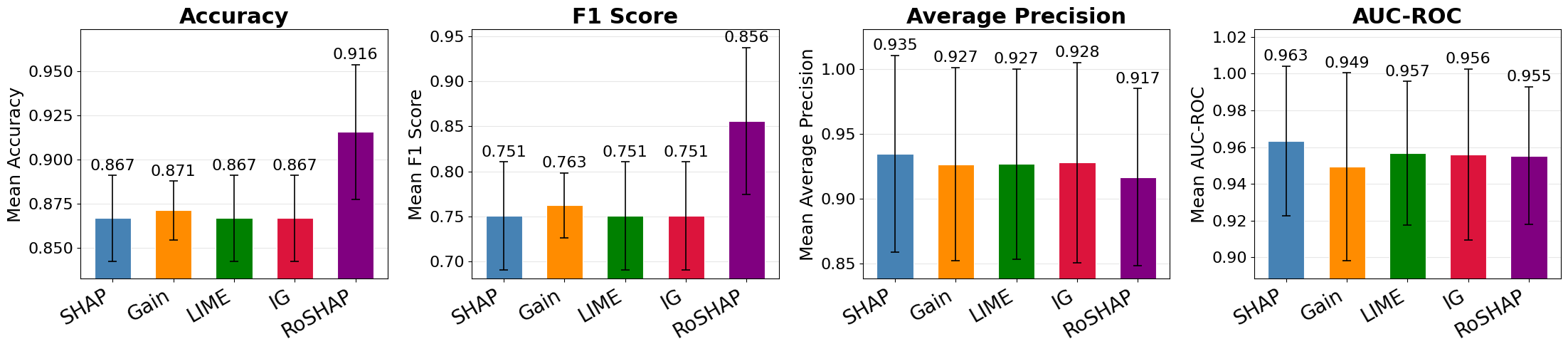}
    \caption{Feature-selection performance on the Golub data using XGBoost. Methods are SHAP (blue), gain (orange), LIME (green), IG (red), and RoSHAP (purple). Bars show mean performance across candidate feature-set sizes; error bars show variation. Higher is better.}
    \label{fig:golub_XGBoost_comparison}
    
\end{figure}
\vspace{-15pt}
\subsection{Classification: Musk (Version 2)}


The Musk (Version 2) dataset \citep{chapman1994musk2} contains 6,598 instances and 166 continuous features. The binary classification task is to distinguish musk from non-musk molecules, where each instance represents one molecular conformation described by physicochemical descriptors. The dataset is imbalanced, with 5,581 non-musk and 1,017 musk instances, making it useful for evaluating feature contribution robustness in high-dimensional molecular classification.

\begin{wrapfigure}[16]{r}{0.33\textwidth}
    \vspace{-2\baselineskip}
    \centering
    \includegraphics[
        width=\linewidth,
    ]{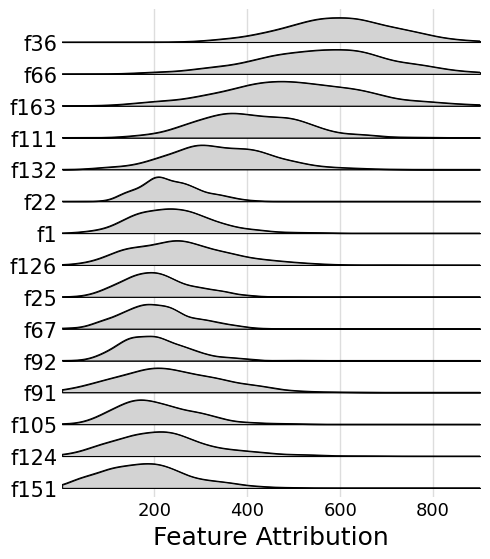}
    \caption{Distribution of top 15 CatBoost feature attribution estimates for the Musk dataset.}
    \label{fig:musk_cb_top15_dist}
\end{wrapfigure}

The Musk experiment studies feature-level attribution using a CatBoost classifier. Figure~\ref{fig:musk_cb_top15_dist} shows the distributions of the top 15 CatBoost feature contribution estimates for the Musk dataset. Several features have similar contribution magnitudes and variability, indicating that the predictive signal is distributed across multiple features rather than dominated by a single feature.

For variable selection, only the top 30 selected features are compared for Musk dataset, since prediction performance remains high beyond this point. Figure~\ref{fig:musk_CatBoost_bar_comparison} shows the feature selection performance for the Musk dataset. Performance across selection metrics is largely comparable. The feature contribution distribution suggests that many features have similar levels of importance. As a result, differences in the selected feature subsets have limited impact on prediction performance.

\begin{figure}[ht]
    \centering
    \includegraphics[width=0.9\linewidth]{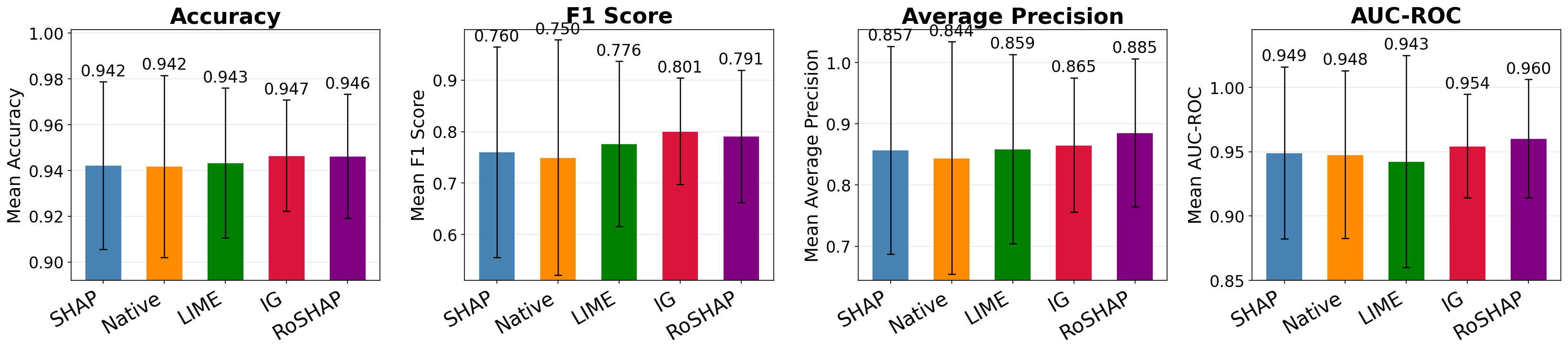}
    \caption{Feature-selection performance on the Musk (Version 2) data using CatBoost. Methods are SHAP (blue), gain (orange), LIME (green), IG (red), and RoSHAP (purple). Bars show mean performance across candidate feature-set sizes; error bars show variation. Higher is better.}
    \label{fig:musk_CatBoost_bar_comparison}
    \vspace{-1\baselineskip}
\end{figure}

\subsection{Regression: UJIIndoorLoc}
The UJIIndoorLoc dataset is used to evaluate the proposed framework in a regression task. It is a public indoor-localization benchmark collected at Universitat Jaume I for evaluating WLAN/WiFi fingerprint-based positioning systems \citep{ujiindoorloc_310}. The dataset contains 21,048 observations, including 19,937 training/reference records and 1,111 validation/test records. Each observation includes received signal strength intensity (RSSI) measurements from 520 wireless access points (WAPs), along with location and metadata variables.

The RSSI values range from approximately \(-104\) to \(0\) dBm, with \(100\) indicating an undetected access point. In this study, the 520 WAP signal-strength variables are used as predictors and longitude is used as the continuous response. The value \(100\) is recoded as \(-110\) to represent a very weak or undetected signal.

\begin{wrapfigure}[16]{r}{0.33\textwidth}
    \vspace{-3em} 
    \centering
    \includegraphics[width=\linewidth]{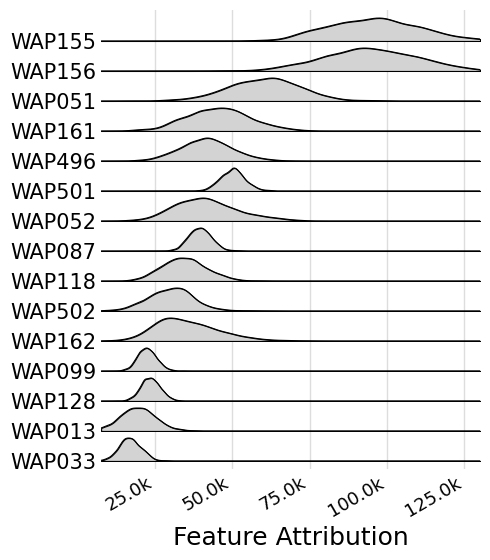}
    \caption{Distribution of top 15 LightGBM feature attribution estimates for the UJIIndoorLoc dataset.}
    \label{fig:uji_LightGBM_top15_dist}
    \vspace{-1em}
\end{wrapfigure}

This experiment studies feature-level attribution using a LightGBM regression model. Figure~\ref{fig:uji_LightGBM_top15_dist} shows the distributions of the top 15 LightGBM feature attribution estimates for the UJIIndoorLoc dataset. The top two features have similar attribution patterns, while the remaining features exhibit different levels of variability: some distributions are more concentrated, whereas others are more dispersed. Figure~\ref{fig:uji_LightGBM_bar_comparison} compares the feature-selection performance of different importance measures on the UJIIndoorLoc regression task using LightGBM. Overall, RoSHAP achieves competitive and stable performance across the evaluation metrics.

\begin{figure}[ht]
    \centering
    \includegraphics[width=0.7\linewidth]{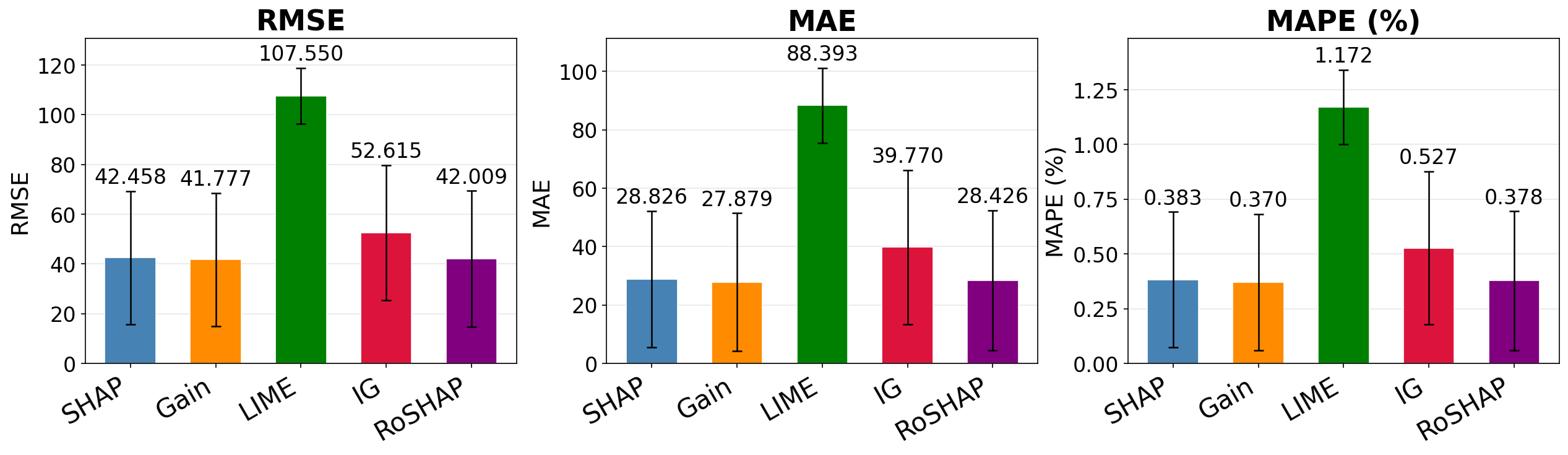}
    \caption{Feature-selection performance on the UJIIndoorLoc data using LightGBM. Methods are SHAP (blue), gain (orange), LIME (green), IG (red), and the RoSHAP (purple). Bars show mean performance across candidate feature-set sizes; error bars show variation. Lower is better.}
    \label{fig:uji_LightGBM_bar_comparison}
    \vspace{-1em}
\end{figure}

\subsection{Classification: CIFAR-10}
This experiment extends the proposed framework beyond tabular data to sample-level image interpretation. The CIFAR-10 dataset contains 60,000 \(32\times 32\) color images from 10 classes, with 50,000 training and 10,000 test images \citep{krizhevsky2009learning}.
We conduct 50 bootstrap-based SHAP analyses using a ViT-base model. 
For each test image, the SHAP array has dimension \(50 \times p \times 10\), where \(p=24\times24\) is the number of image patches and 10 is the number of classes.
Results focus on the SHAP distribution for the predicted class.

As an illustrative example, one ship image that was correctly predicted by most bootstrap-fitted models is selected.
Figure~\ref{fig:ship_9round_shap} shows patch-level SHAP maps from three randomly selected bootstrap runs, where red indicates positive SHAP values, blue negative values, and white values near zero.  
Although the prediction is accurate, the contributing patches vary noticeably across runs, suggesting that single-run image explanations are sensitive to training randomness. 
For example, in Round 1 at For example, the largest SHAP value appears on the water region in Round 1 but shifts toward the ship body in Round 3.  
After aggregation, the robust attribution map in Figure~\ref{fig:cifar_correct} shows that the prediction is mainly attributed to patches on the ship body.
\begin{wrapfigure}[8]{r}{0.35\textwidth}
\vspace{-1em}
    \centering
    \includegraphics[width=\linewidth]{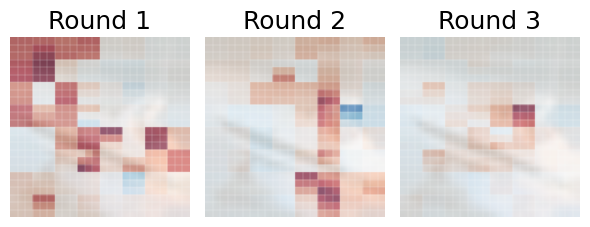}
    \caption{CIFAR-10 ship example using 3 bootstrap rounds.}
    \label{fig:ship_9round_shap}
\end{wrapfigure}

The proposed framework provides additional insight into attribution patterns behind image predictions. Figure~\ref{fig:cifar_explanation_examples} illustrates two representative cases: a correct prediction with attribution focused on the target object, and an incorrect prediction in which attribution is concentrated on irrelevant regions. From left to right, each row shows the original image, the median SHAP map for the predicted class, the standard deviation of SHAP values across bootstrap runs, and the robust SHAP map. In the median SHAP map, red indicates positive contribution, blue negative contribution, and white near-zero contribution; in the standard deviation map, red indicates larger variation and white little variation; in the robust SHAP map, green indicates larger contribution and white no contribution. In Figure~\ref{fig:cifar_correct}, the model focuses mainly on the correct object region. In Figure~\ref{fig:cifar_wrong_place}, however, the model focuses largely on the sky rather than the car, leading to an incorrect bird prediction. These examples show that the framework can help diagnose prediction errors by revealing whether the model attends to the appropriate image regions.

\begin{figure}[ht]
    \centering
    \begin{subfigure}{0.7\linewidth}
        \centering
        \includegraphics[width=\linewidth]{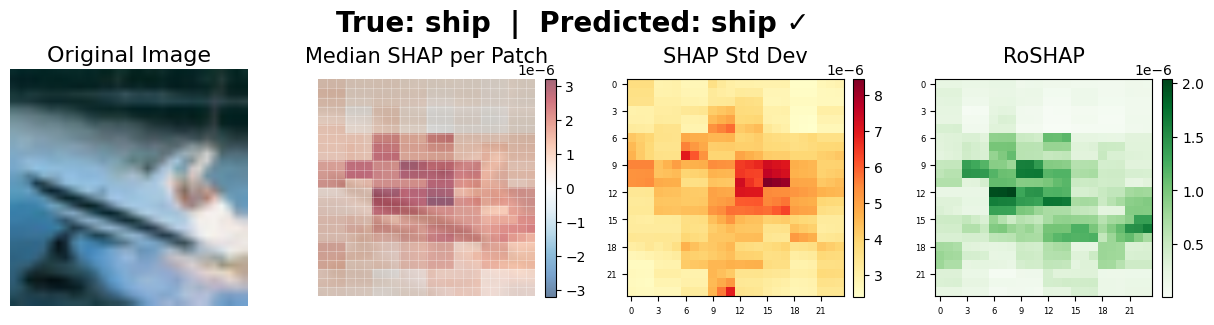}
        \caption{Correct prediction; attribution focuses on the ship.}
        \label{fig:cifar_correct}
    \end{subfigure}

    \begin{subfigure}{0.7\linewidth}
        \centering
        \includegraphics[width=\linewidth]{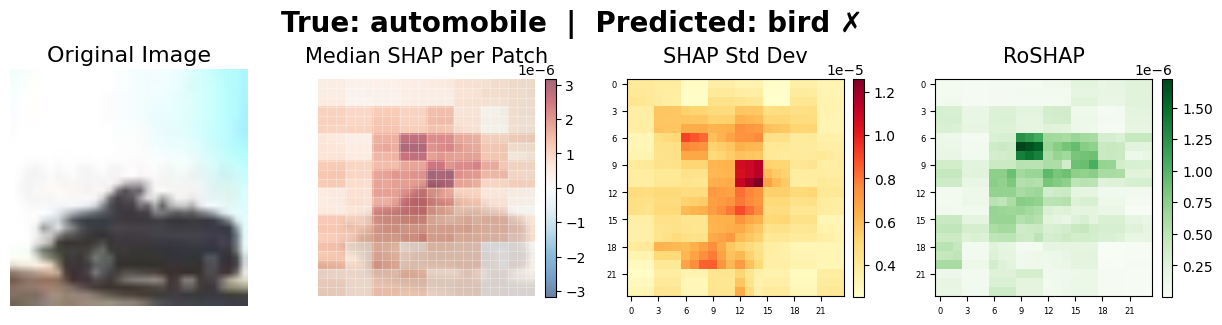}
        \caption{Wrong prediction; attribution does not focus on the car.}
        \label{fig:cifar_wrong_place}
    \end{subfigure}

    \caption{CIFAR-10 explanations after 50 ViT-base bootstrap runs. Columns show the original image, median SHAP map, SHAP standard deviation, and robust SHAP map. Red/blue denote positive/negative SHAP values, red in the standard deviation map denotes larger variability, and green denotes larger robust attribution.}
    \label{fig:cifar_explanation_examples}
\end{figure}

\section{Discussion}
\label{sec:discussion}
This study addresses a key limitation of existing feature attribution methods: attribution scores are often treated as fixed quantities despite their stochastic nature.
We propose a distributional framework for characterizing attribution uncertainty and introduce RoSHAP, a robust SHAP-based metric for stable feature ranking and contribution evaluation.
RoSHAP prioritizes features that are consistently active, strong, and stable across repeated analyses.
We extensively test our model on realistic datasets and demonstrate that the proposed robust metric improves interpretation stability, helps distinguish consistent signals from chance-selected features, and can be applied across different data types. 
Although the bootstrap procedure remains more costly than single-run attribution, the Gaussian approximation helps extensively reduce this burden when applicable. 
Future work can extend the framework from data-sampling randomness to model-level randomness.
Overall, RoSHAP serves as a useful complement to existing feature attribution methods by supporting more stable and reliable interpretation.




\clearpage
{\small
\bibliographystyle{unsrtnat}
\bibliography{references}
}


\appendix

\section{Appendix / supplemental material}


\subsection{Proof of Theorem}
\label{app:proof_clt}

\begin{lemma}[Lyapunov central limit theorem {\cite[Theorem~27.3]{billingsley1995probability}}]
Let \(Z_{n1},\dots,Z_{nn}\) be independent random variables with
\(E[Z_{ni}]=0\), variances \(\sigma_{ni}^2\), and
\(s_n^2=\sum_{i=1}^n \sigma_{ni}^2\). If, for some \(\delta>0\),
\[
\frac{1}{s_n^{2+\delta}}
\sum_{i=1}^n E|Z_{ni}|^{2+\delta}\to 0,
\]
then
\[
\frac{\sum_{i=1}^n Z_{ni}}{s_n}\xrightarrow{d}N(0,1).
\]
\end{lemma}

\begin{proof}
Since
\[
|T_{ij}|=(1-B_{ij})H_{ij},
\]
the independence of \(B_{ij}\) and \(H_{ij}\) gives
\[
E(|T_{ij}|)
=E(1-B_{ij})E(H_{ij})
=(1-w_{ij})E(H_{ij}).
\]
Therefore,
\[
\mu_j=E(U_j)
=\sum_{i=1}^n E(|T_{ij}|)
=\sum_{i=1}^n (1-w_{ij})E(H_{ij}).
\]

Next, by the law of total variance,
\[
\operatorname{Var}(|T_{ij}|)
=
E(\operatorname{Var}(|T_{ij}|\mid B_{ij}))
+
\operatorname{Var}(E(|T_{ij}|\mid B_{ij})).
\]
Conditionally on \(B_{ij}\),
\[
E(|T_{ij}|\mid B_{ij})=(1-B_{ij})E(H_{ij}),
\qquad
\operatorname{Var}(|T_{ij}|\mid B_{ij})
=(1-B_{ij})\operatorname{Var}(H_{ij}).
\]
Hence
\[
\operatorname{Var}(|T_{ij}|)
=
(1-w_{ij})\operatorname{Var}(H_{ij})
+
w_{ij}(1-w_{ij})\{E(H_{ij})\}^2.
\]
Since the pairs \((B_{ij},H_{ij})\) are mutually independent, the variables
\(|T_{ij}|\) are mutually independent, and therefore
\[
s_j^2
=
\operatorname{Var}(U_j)
=
\sum_{i=1}^n \operatorname{Var}(|T_{ij}|)
=
\sum_{i=1}^n (1-w_{ij})\operatorname{Var}(H_{ij})
+
\sum_{i=1}^n w_{ij}(1-w_{ij})\{E(H_{ij})\}^2.
\]

Now define the centered summands
\[
Z_{ni}:=|T_{ij}|-E(|T_{ij}|).
\]
By the stated Lyapunov condition with \(\delta>0\),
\[
\frac{1}{s_j^{2+\delta}}
\sum_{i=1}^n
E\left|
|T_{ij}|-E(|T_{ij}|)
\right|^{2+\delta}
\to 0.
\]
Thus Lyapunov's central limit theorem applies, giving
\[
\frac{\sum_{i=1}^n \{|T_{ij}|-E(|T_{ij}|)\}}{s_j}
\xrightarrow{d}N(0,1).
\]
Since \(U_j=\sum_{i=1}^n |T_{ij}|\) and
\(\mu_j=\sum_{i=1}^n E(|T_{ij}|)\), this is exactly
\[
\frac{U_j-\mu_j}{s_j}
\xrightarrow{d}N(0,1).
\]
\end{proof}

\begin{remark}[On the Lyapunov condition]
A natural choice is \(\delta=1\), which reduces the condition to
\[
\frac{1}{s_j^{3}}
\sum_{i=1}^n
E\bigl|
|T_{ij}|-E(|T_{ij}|)
\bigr|^{3}
\to 0.
\]

A simple sufficient condition is that the third moments are uniformly bounded,
\(\sup_i E|T_{ij}|^3 < \infty\), and no single term dominates the variance, i.e.,
\[
\max_{1\le i\le n} \operatorname{Var}(|T_{ij}|) = o(s_j^2).
\]
Under these conditions, the numerator grows at most linearly in \(n\), while
\(s_j^2\) grows proportionally to the number of contributing terms, yielding
\[
\frac{n}{s_j^3}\to 0,
\]
and hence the Lyapunov condition holds.
\end{remark}

\subsection{Additional Golub Data Results}
\label{app:golub}

We employed the same six classifiers as in the previous experiments, with hyperparameters re-tuned to suit the characteristics of the Golub dataset. The  tree-based ensemble models were configured as follows: LightGBM (binary objective, binary logloss metric, learning rate 0.1, max depth 8, 31 leaves, minimum 15 samples per leaf, 80\% feature and bagging fractions with bagging frequency of 1, 150 boosting rounds); XGBoost (binary logistic objective, logloss evaluation metric, learning rate 0.1, max depth 6, 80\% subsample and 60\% column sampling, 100 boosting rounds); and CatBoost (logloss metric, learning rate 0.1, depth 6, default iterations); Random Forest used 100 trees with max depth 6, while Gradient Boosting used 200 trees with max depth 3 at the default learning rate. As a linear baseline, Logistic Regression was trained with a maximum of 200 iterations using default regularization. All models were initialized with a fixed random seed to ensure reproducibility, and feature contribution was extracted using tree-based explainers for the ensemble models and a linear (coefficient-based) explainer for Logistic Regression. The XGBoost results are presented in the main text, while the feature selection performance for the remaining classifiers on the Golub dataset is shown in Figures~\ref{fig:golub_CatBoost_comparison}, \ref{fig:golub_GradientBoosting_comparison}, \ref{fig:golub_LightGBM_comparison}, \ref{fig:golub_LogisticRegression_comparison}, and \ref{fig:golub_RandomForest_comparison}.

\begin{figure}[H]
    \centering
    \includegraphics[width=0.9\linewidth]{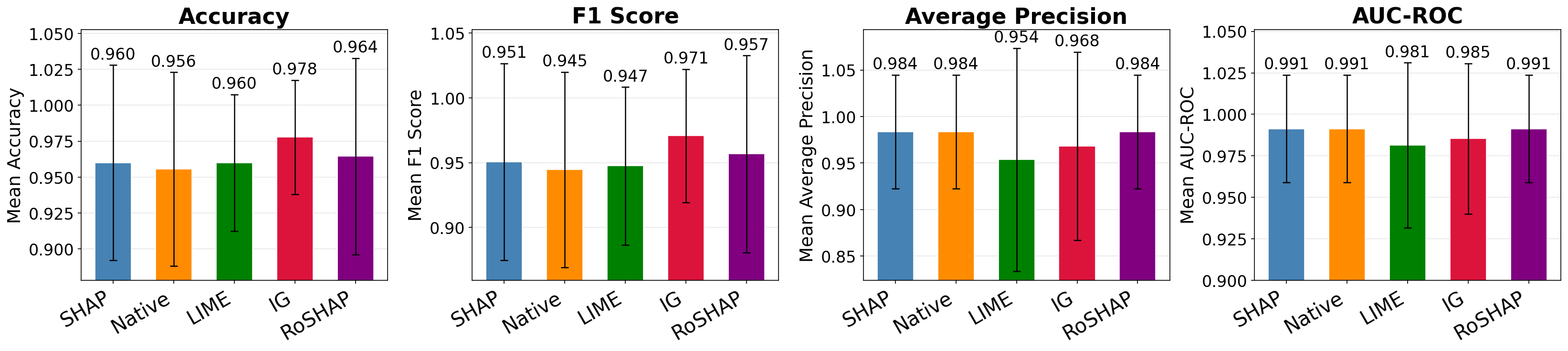}
    \caption{Golub gene selection performance comparison using CatBoost. Feature selection methods include SHAP (blue), gain (orange), LIME (green), information gain (IG; red), and the proposed robust method (purple). Higher better. }
    \label{fig:golub_CatBoost_comparison}
\end{figure}

\vspace{-1em}
\begin{figure}[H]
    \centering
    \includegraphics[width=0.9\linewidth]{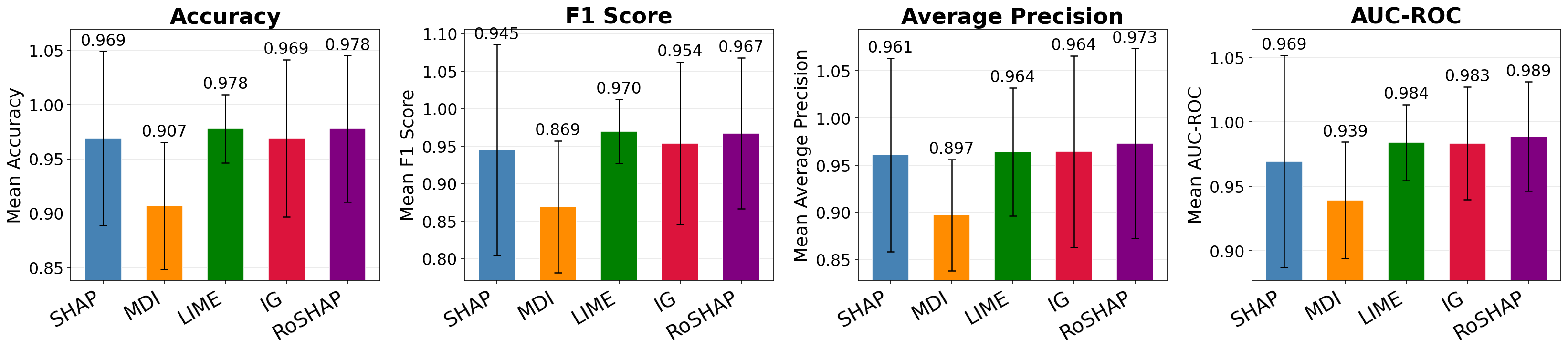}
    \caption{Golub gene selection performance comparison using Gradient Boosting. Feature selection methods include SHAP (blue), gain (orange), LIME (green), information gain (IG; red), and the proposed robust method (purple). Higher better. }
    \label{fig:golub_GradientBoosting_comparison}
\end{figure}

\vspace{-2em}
\begin{figure}[H]
    \centering
    \includegraphics[width=0.9\linewidth]{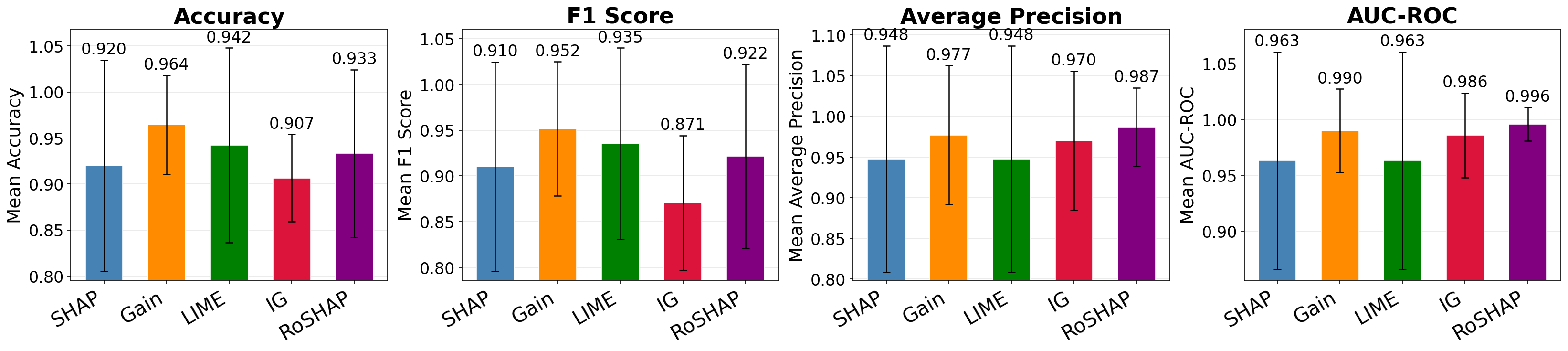}
    \caption{Golub gene selection performance comparison using LightGBM. Feature selection methods include SHAP (blue), gain (orange), LIME (green), information gain (IG; red), and the proposed robust method (purple). Higher better. }
    \label{fig:golub_LightGBM_comparison}
\end{figure}

\vspace{-2em}
\begin{figure}[H]
    \centering
    \includegraphics[width=0.9\linewidth]{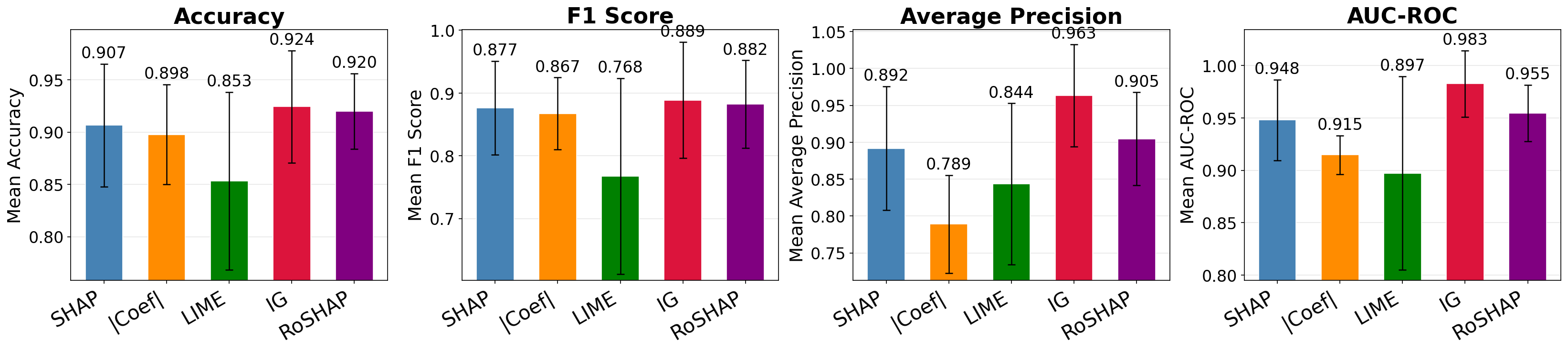}
    \caption{Golub gene selection performance comparison using Logistic Regression. Feature selection methods include SHAP (blue), |coefficient| (orange), LIME (green), information gain (IG; red), and the proposed robust method (purple). Higher better. }
    \label{fig:golub_LogisticRegression_comparison}
\end{figure}

\vspace{-2em}
\begin{figure}[H]
    \centering
    \includegraphics[width=0.9\linewidth]{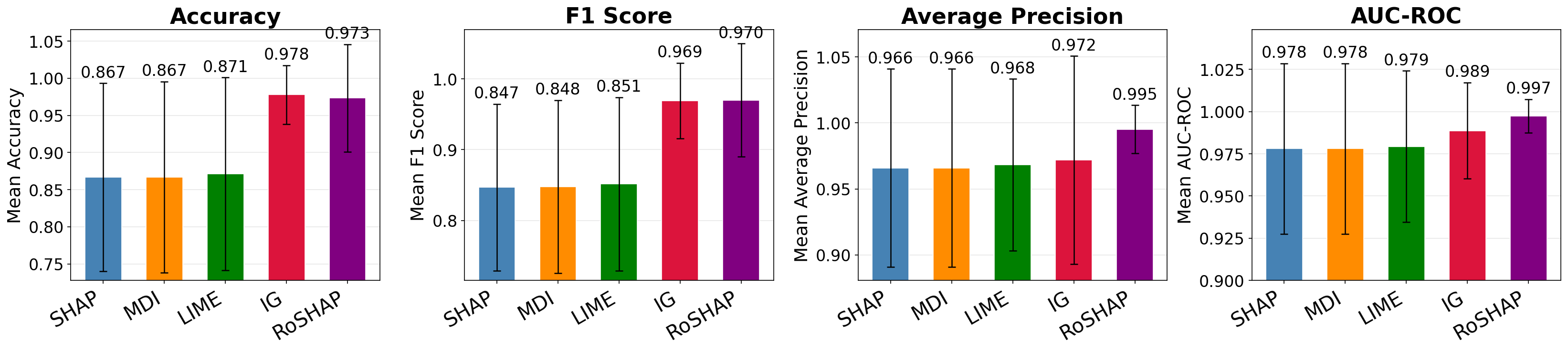}
    \caption{Golub gene selection performance comparison using Random Forest. Feature selection methods include SHAP (blue), gain (orange), LIME (green), information gain (IG; red), and the proposed robust method (purple). Higher better.}
    \label{fig:golub_RandomForest_comparison}
\end{figure}

\subsection{Additional Musk (Version 2) Data Results}
\label{app:musk}
The same six classifiers in the previous experiments were used, with hyperparameters re-tuned to suit the characteristics of the Musk (Version 2) dataset. The  tree-based ensemble models were configured as follows: LightGBM (binary objective, binary logloss metric, learning rate 0.1, max depth 8, 31 leaves, minimum 15 samples per leaf, 80\% feature and bagging fractions with bagging frequency of 1, 150 boosting rounds); XGBoost (binary logistic objective, logloss evaluation metric, learning rate 0.1, max depth 6, 80\% subsample and 60\% column sampling, 100 boosting rounds); and CatBoost (logloss metric, learning rate 0.1, depth 6, default iterations); Random Forest used 100 trees with max depth 6, while Gradient Boosting used 200 trees with max depth 3 at the default learning rate. Logistic Regression was trained with a maximum of 200 iterations using default regularization. All models were initialized with a fixed random seed to ensure reproducibility, and feature contribution was extracted using tree-based explainers for the ensemble models and a linear (coefficient-based) explainer for Logistic Regression. The CatBoost results are presented in the main text, while the feature selection performance for the remaining classifiers on the Musk Version 2 dataset is shown in Figures~\ref{fig:musk_XGBoost_comparison}, \ref{fig:musk_GradientBoosting_comparison}, \ref{fig:musk_LightGBM_comparison}, \ref{fig:musk_LogisticRegression_comparison}, and \ref{fig:musk_RandomForest_comparison}.

\begin{figure}[H]
    \centering
    \includegraphics[width=0.9\linewidth]{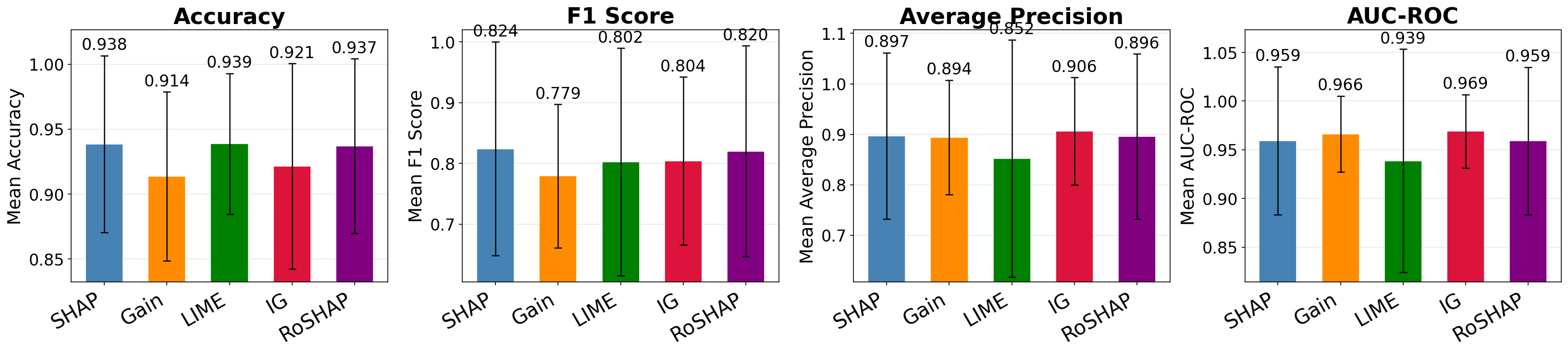}
    \caption{Musk (Version 2) variable selection performance comparison using XGBoost. Feature selection methods include SHAP (blue), gain (orange), LIME (green), information gain (IG; red), and the proposed robust method (purple).}
    \label{fig:musk_XGBoost_comparison}
\end{figure}

\vspace{-2em}
\begin{figure}[H]
    \centering
    \includegraphics[width=0.9\linewidth]{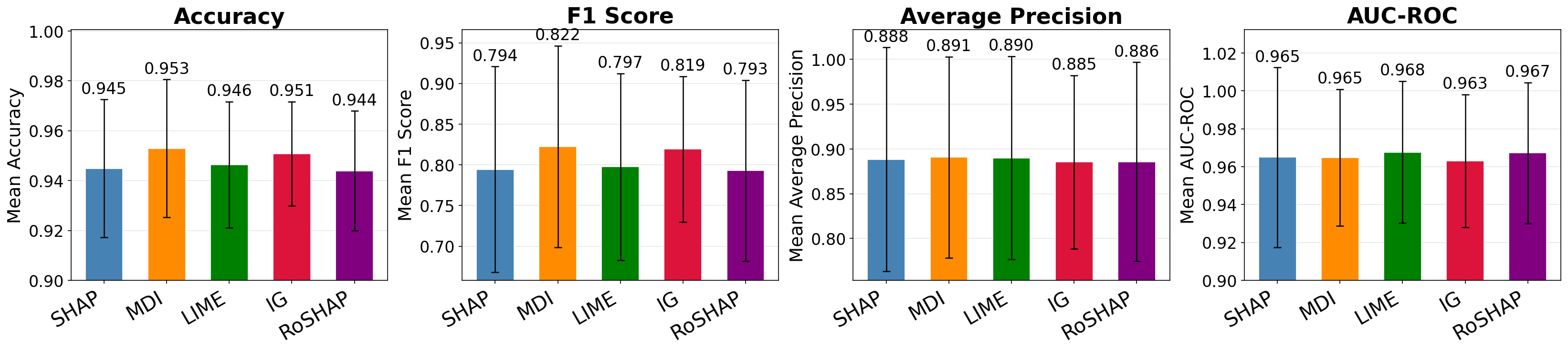}
    \caption{Musk (Version 2) variable selection performance comparison using Gradient Boosting. Feature selection methods include SHAP (blue), gain (orange), LIME (green), information gain (IG; red), and the proposed robust method (purple).}
    \label{fig:musk_GradientBoosting_comparison}
\end{figure}

\vspace{-2em}
\begin{figure}[H]
    \centering
    \includegraphics[width=0.9\linewidth]{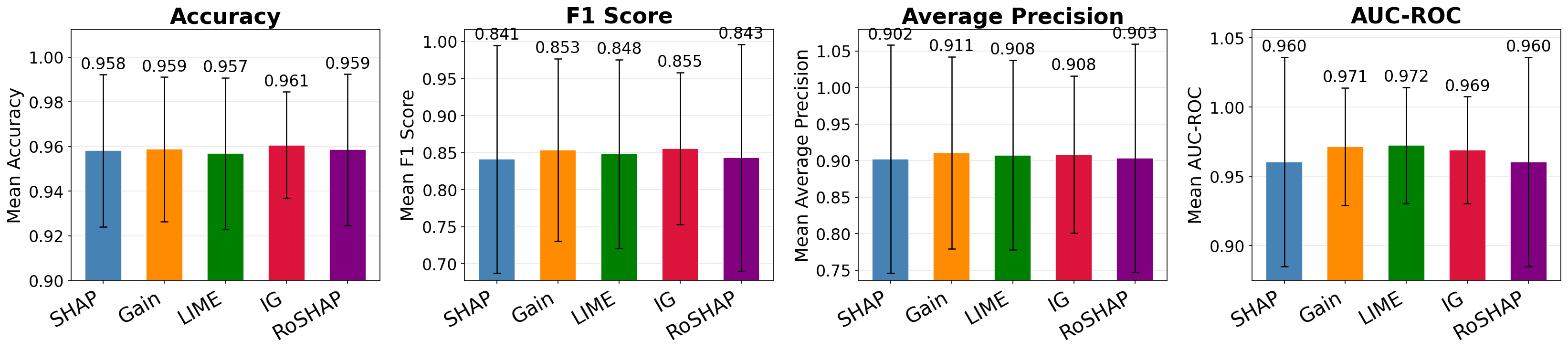}
    \caption{Musk (Version 2) variable selection performance comparison using LightGBM. Feature selection methods include SHAP (blue), gain (orange), LIME (green), information gain (IG; red), and the proposed robust method (purple).}
    \label{fig:musk_LightGBM_comparison}
\end{figure}

\vspace{-2em}
\begin{figure}[H]
    \centering
    \includegraphics[width=0.9\linewidth]{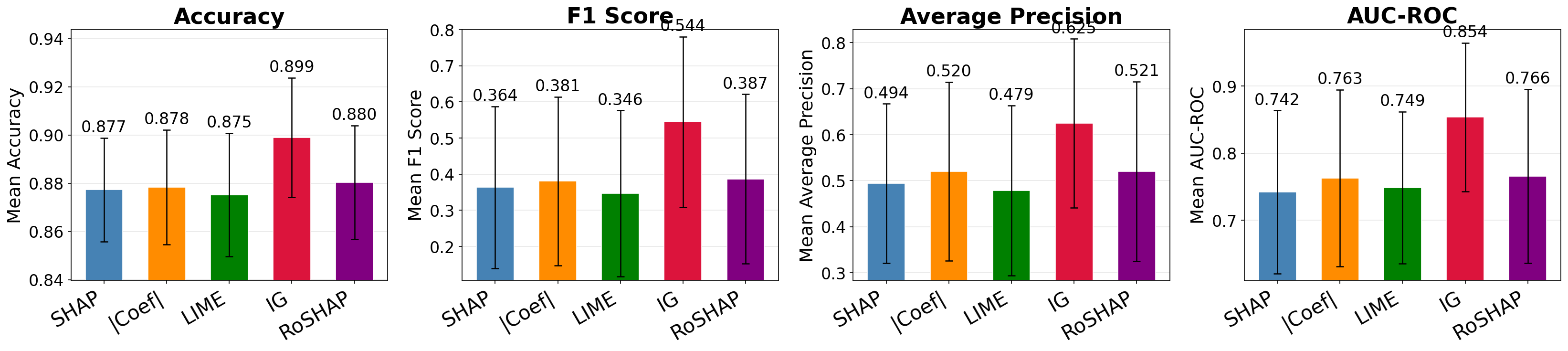}
    \caption{Musk (Version 2) variable selection performance comparison using Logistic Regression. Feature selection methods include SHAP (blue), |coefficient| (orange), LIME (green), information gain (IG; red), and the proposed robust method (purple).}
    \label{fig:musk_LogisticRegression_comparison}
\end{figure}
\vspace{-2em}
\begin{figure}[H]
    \centering
    \includegraphics[width=0.9\linewidth]{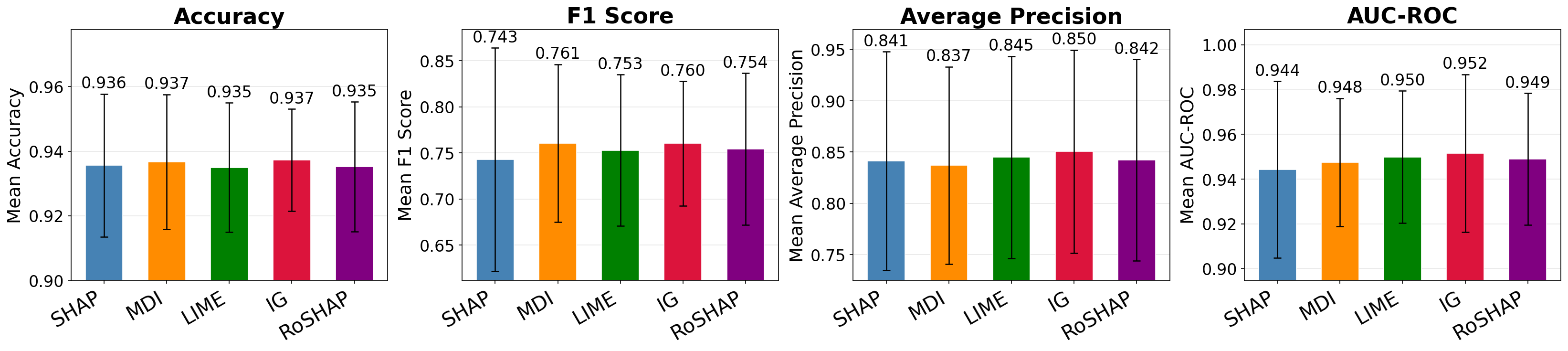}
    \caption{Musk (Version 2) variable selection performance comparison using Random Forest. Feature selection methods include SHAP (blue), gain (orange), LIME (green), information gain (IG; red), and the proposed robust method (purple).}
    \label{fig:musk_RandomForest_comparison}
\end{figure}

\subsection{Additional UJIIndoorLoc Results}
For the UJIIndoorLoc regression experiment, except for the random forest model, same predictive models used in the previous experiments were considered, with hyperparameters re-tuned to suit the characteristics of the UJIIndoorLoc dataset. The response variable was longitude. The tree-based ensemble models were configured as follows: LightGBM used a regression objective with RMSE metric, learning rate 0.03, max depth 7, 63 leaves, minimum 30 samples per leaf, 70\% feature fraction, 80\% bagging fraction with bagging frequency of 1, $\ell_1$ regularization 0.01, $\ell_2$ regularization 1.0, and 300 boosting rounds; XGBoost used a squared-error regression objective with RMSE evaluation metric, learning rate 0.03, max depth 6, 80\% subsample ratio, 70\% column sampling, minimum child weight 10, $\ell_1$ regularization 0.01, $\ell_2$ regularization 2.0, histogram-based tree construction, and 300 boosting rounds; and CatBoost used RMSE loss and evaluation metric, learning rate 0.03, depth 6, 300 iterations, $\ell_2$ leaf regularization 5.0, and 70\% random subspace sampling. Ridge Regression was included as a linear baseline with regularization parameter $\alpha=10$. All models were initialized with a fixed random seed to ensure reproducibility, and feature contribution was extracted using tree-based explainers for the ensemble models and a linear explainer for Ridge Regression.
The LightGBM results are presented in the main text, while the feature selection performance for the remaining models on the UJIIndoorLoc dataset is shown in Figures~\ref{fig:uji_CatBoost_bar_comparison}, \ref{fig:uji_GradientBoosting_bar_comparison}, \ref{fig:uji_XGBoost_bar_comparison}, and \ref{fig:uji_Ridge_bar_comparison}.

\begin{figure}[H]
    \centering
    \includegraphics[width=0.7\linewidth]{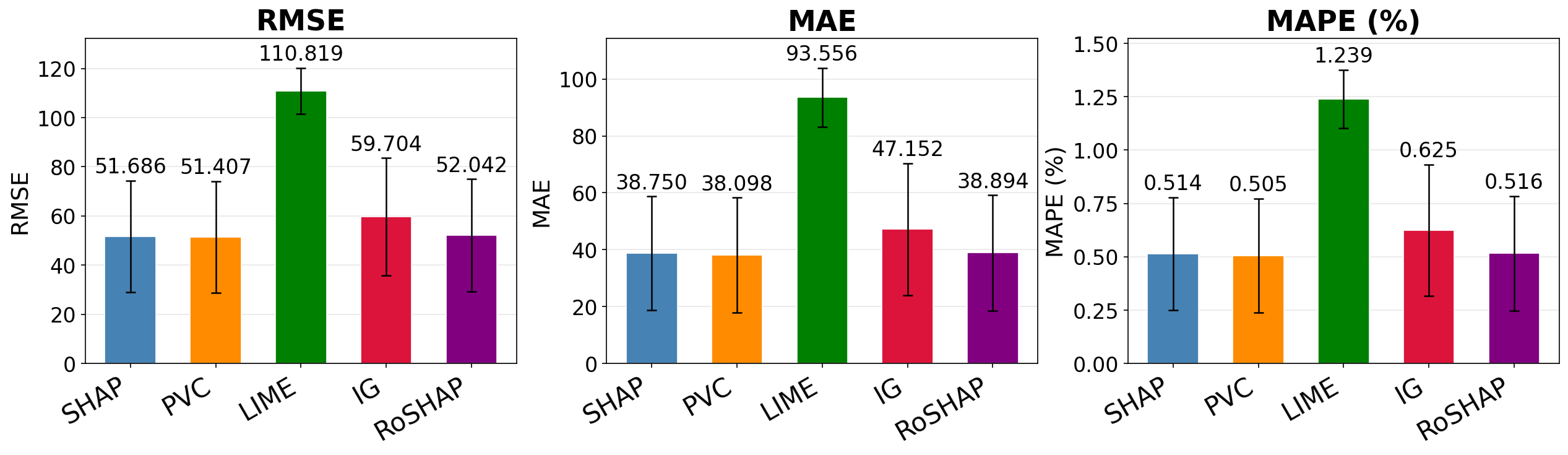}
    \caption{UJIIndoorLoc variable selection performance comparison using CatBoost. Feature selection methods include SHAP (blue), gain (orange), LIME (green), information gain (IG; red), and the proposed robust method (purple). Bar heights represent the mean performance metric across different candidate feature-set sizes, and error bars indicate the corresponding variation. Lower better.}
    \label{fig:uji_CatBoost_bar_comparison}
\end{figure}

\vspace{-2em}
\begin{figure}[H]
    \centering
    \includegraphics[width=0.7\linewidth]{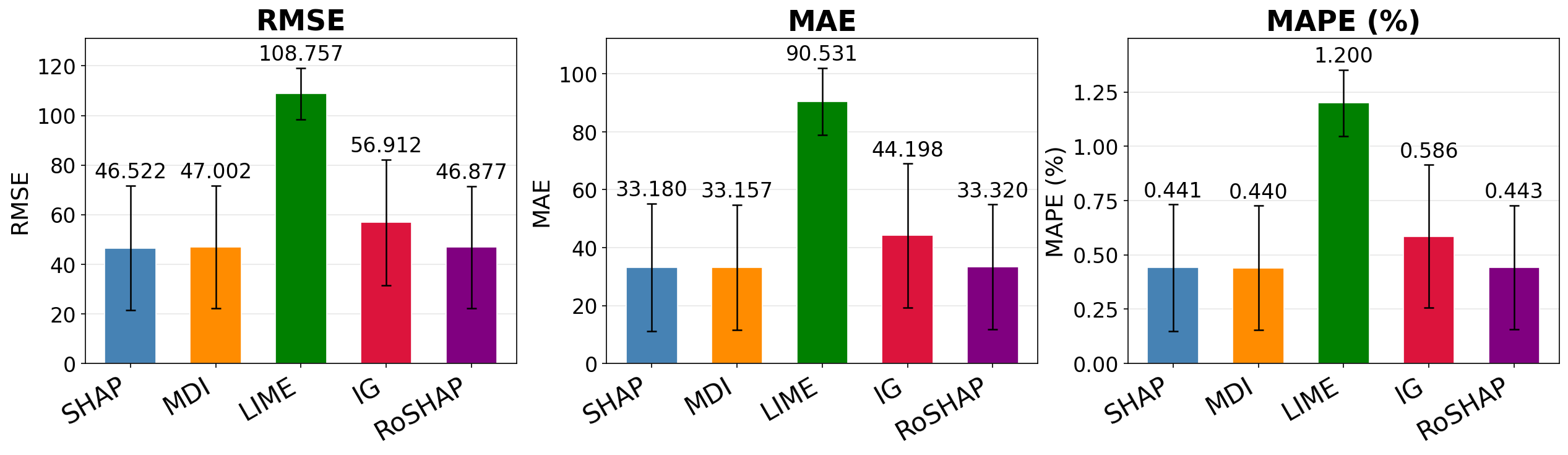}
    \caption{UJIIndoorLoc variable selection performance comparison using Gradient Boosting. Feature selection methods include SHAP (blue), gain (orange), LIME (green), information gain (IG; red), and the proposed robust method (purple). Bar heights represent the mean performance metric across different candidate feature-set sizes, and error bars indicate the corresponding variation. Lower better.}
\label{fig:uji_GradientBoosting_bar_comparison}
\end{figure}

\vspace{-2em}
\begin{figure}[H]
    \centering
    \includegraphics[width=0.7\linewidth]{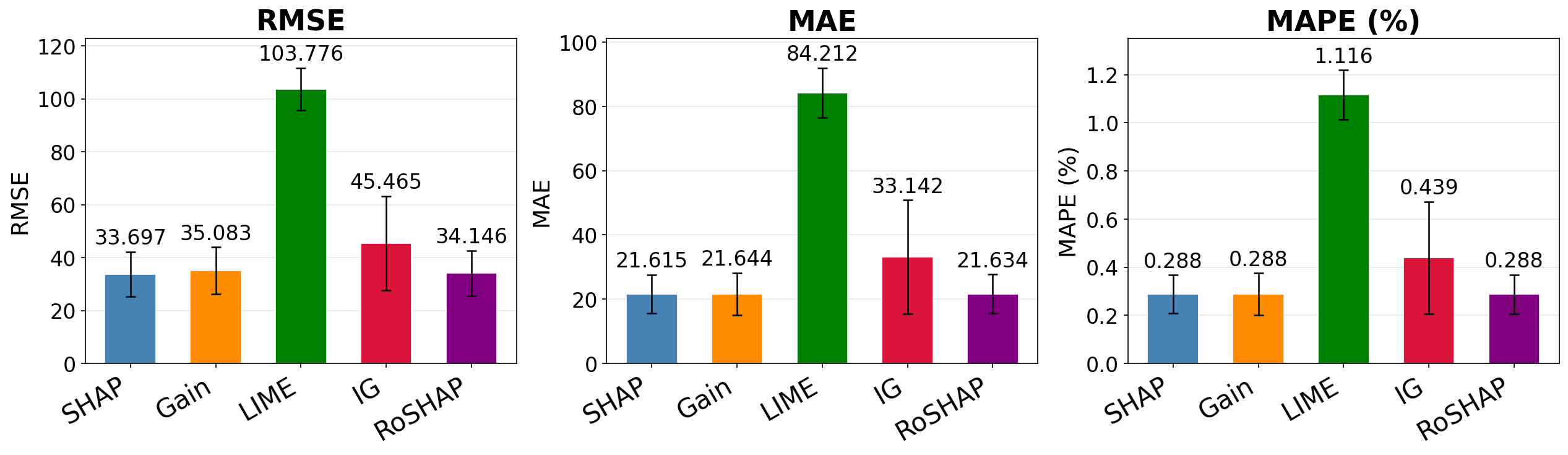}
    \caption{UJIIndoorLoc variable selection performance comparison using XGBoost. Feature selection methods include SHAP (blue), gain (orange), LIME (green), information gain (IG; red), and the proposed robust method (purple). Bar heights represent the mean performance metric across different candidate feature-set sizes, and error bars indicate the corresponding variation. Lower better.}
\label{fig:uji_XGBoost_bar_comparison}
\end{figure}

\vspace{-2em}
\begin{figure}[H]
    \centering
    \includegraphics[width=0.7\linewidth]{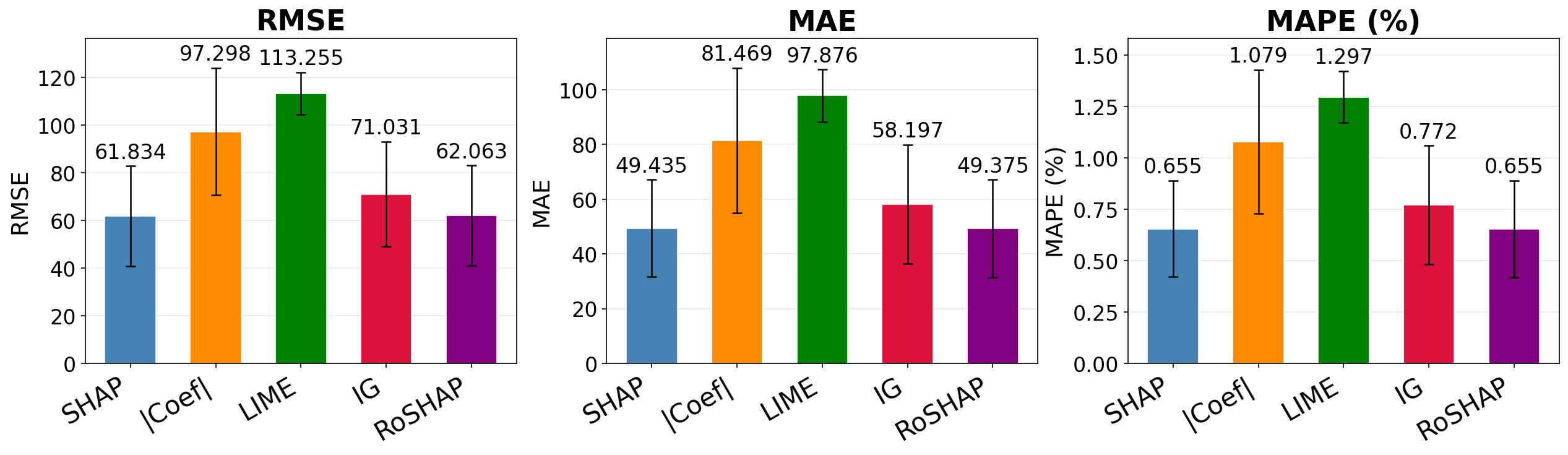}
    \caption{UJIIndoorLoc variable selection performance comparison using Linear Regressio (Ridge). Feature selection methods include SHAP (blue), |coef| (orange), LIME (green), information gain (IG; red), and the proposed robust method (purple). Bar heights represent the mean performance metric across different candidate feature-set sizes, and error bars indicate the corresponding variation. Lower better.}
\label{fig:uji_Ridge_bar_comparison}
\end{figure}

\subsection{Additional CIFAR 10 Data Results}
\label{app:cifar10}

This experiment employ ViT-Base/16, pretrained on ImageNet-21k. The model takes 384×384 inputs divided into 16×16 patches (24×24 grid, 576 patches total), with 12 transformer layers, 12 attention heads, and a hidden dimension of 768 (~86M parameters). The classification head is replaced with a linear layer mapping to 10 CIFAR-10 classes and fine-tuned for 10 epochs using Adam (lr=1e-4, cosine decay) with batch size 128.

\begin{figure}[H]
    \centering
    \includegraphics[width=\linewidth]{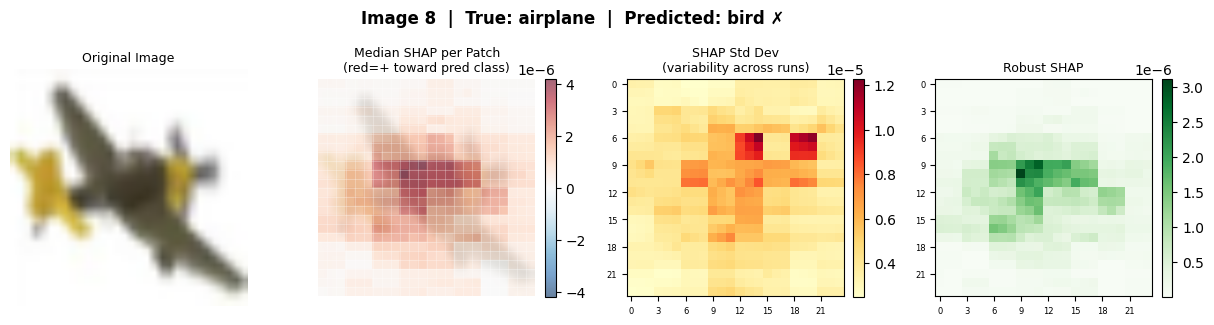}
    \caption{Prediction and bootstrap results for image 8. }
    \label{fig:img_008_explanation}
\end{figure}

\vspace{-2em}
\begin{figure}[H]
    \centering
    \includegraphics[width=\linewidth]{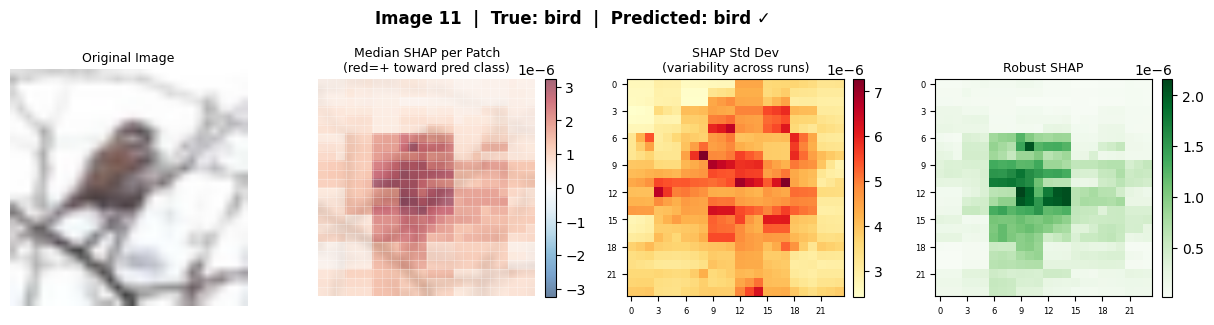}
    \caption{Prediction and bootstrap results for image 11. }
    \label{fig:img_011_explanation}
\end{figure}

\vspace{-2em}
\begin{figure}[H]
    \centering
    \includegraphics[width=\linewidth]{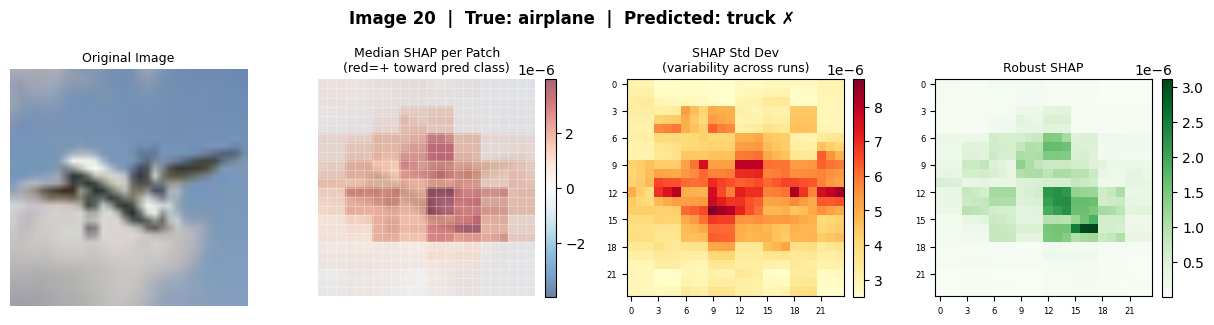}
    \caption{Prediction and bootstrap results for image 20. }
    \label{fig:img_020_explanation}
\end{figure}

\vspace{-2em}
\begin{figure}[H]
    \centering
    \includegraphics[width=\linewidth]{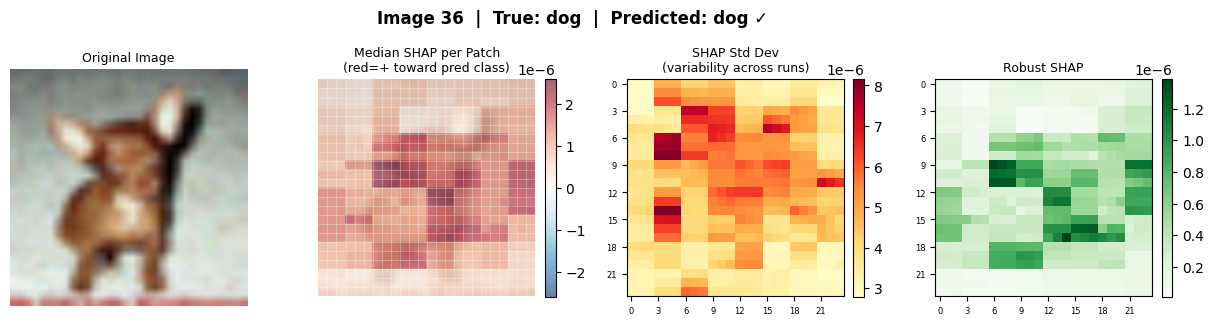}
    \caption{Prediction and bootstrap results for image 36. }
    \label{fig:img_036_explanation}
\end{figure}

\vspace{-2em}
\begin{figure}[H]
    \centering
    \includegraphics[width=\linewidth]{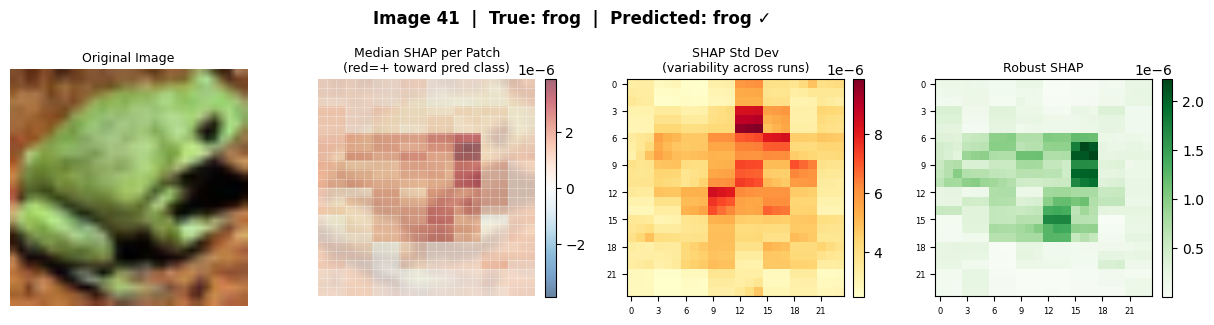}
    \caption{Prediction and bootstrap results for image 41. }
    \label{fig:img_041_explanation}
\end{figure}

\vspace{-2em}
\begin{figure}[H]
    \centering
    \includegraphics[width=\linewidth]{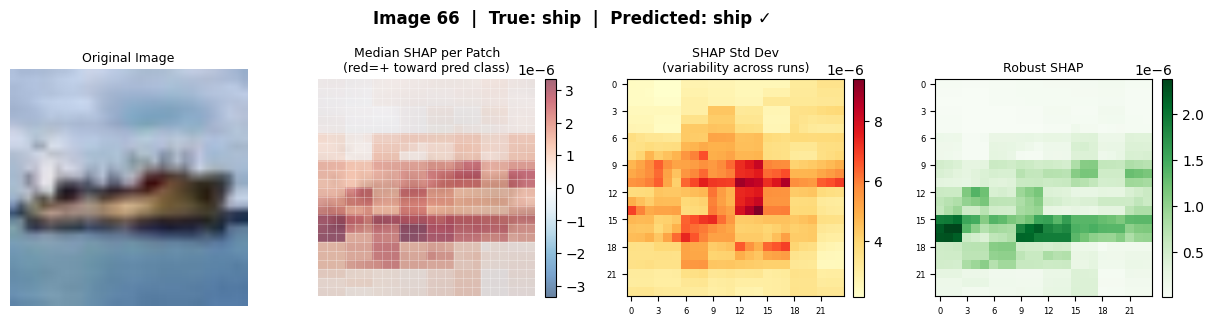}
    \caption{Prediction and bootstrap results for image 66. }
    \label{fig:img_066_explanation}
\end{figure}

\vspace{-2em}
\begin{figure}[H]
    \centering
    \includegraphics[width=\linewidth]{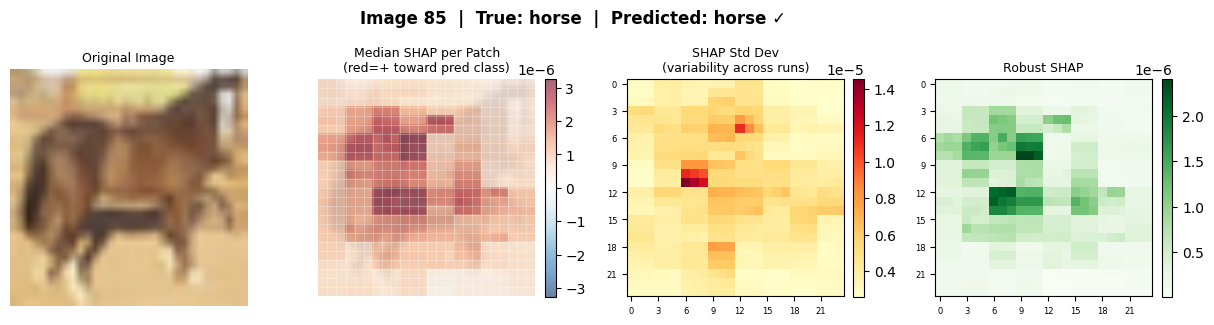}
    \caption{Prediction and bootstrap results for image 85. }
    \label{fig:img_085_explanation}
\end{figure}

\end{document}